\documentclass{article}

\PassOptionsToPackage{numbers, compress}{natbib}
    \usepackage[preprint]{neurips_2022}

\usepackage[utf8]{inputenc} 
\usepackage[T1]{fontenc}    
\usepackage{hyperref}       
\usepackage{url}            
\usepackage{booktabs}       
\usepackage{amsfonts}       
\usepackage{nicefrac}       
\usepackage{microtype}      
\usepackage{xcolor}         
\usepackage{amsmath}
\usepackage{amssymb}
\usepackage{mathtools}
\usepackage{amsthm}
\usepackage{bm}
\usepackage{subfig}
\usepackage{hyperref}
\usepackage{graphicx}
\usepackage{subfig}
\usepackage{algorithm}
\usepackage{algorithmic}
\usepackage{enumerate}

\theoremstyle{plain}
\newtheorem{theorem}{Theorem}[section]
\newtheorem{proposition}[theorem]{Proposition}
\newtheorem{lemma}[theorem]{Lemma}

\theoremstyle{definition}

\newtheorem{assumption}[theorem]{Assumption}
\theoremstyle{remark}
\newtheorem{remark}[theorem]{Remark}

\newenvironment{itemize*}{
\begin{itemize}
\setlength{\parskip}{0em}
\setlength{\topparskip}{0em}
}
{\end{itemize}}

\newenvironment{enumerate*}{
\begin{enumerate}
\setlength{\parskip}{0em}
\setlength{\topparskip}{0em}
}
{\end{enumerate}}

\newcommand{\ssam}[1]{}  
\newcommand{\mmp}[1]{}

\newcommand{\tppalg}{{\sc TangentSpaceBasis}}
\newcommand{\ouralg}{{\sc ManifoldLasso}}
\newcommand{\tsalg}{{\sc TSLasso}}

\newcommand{\beq}{\begin{equation}}
\newcommand{\eeq}{\end{equation}}
\newcommand{\beqa}{\begin{eqnarray}}
\newcommand{\eeqa}{\end{eqnarray}}
\newcommand{\beqas}{\begin{eqnarray*}}
\newcommand{\eeqas}{\end{eqnarray*}}

\newcommand{\bit}{\begin{itemize}}
\newcommand{\eit}{\end{itemize}}
\newcommand{\bits}{\begin{itemize*}}
\newcommand{\eits}{\end{itemize*}}
\newcommand{\benum}{\begin{enumerate}}
\newcommand{\eenum}{\end{enumerate}}
\newcommand{\benums}{\begin{enumerate*}}
\newcommand{\eenums}{\end{enumerate*}}

\newcommand{\norm}[1]{\lvert\lvert{#1}\lvert\lvert}

\newcommand{\dataset}{{\cal D}}

\newcommand{\rrr}{{\mathbb R}}

\newcommand{\M}{{\cal M}}
\newcommand{\T}{{\cal T}}
\newcommand{\I}{{\cal I}} 
\newcommand{\F}{{\cal F}}
\newcommand{\G}{{\cal F}} 
\newcommand{\neigh}{{\cal N}}

\newcommand{\xb}{\mathbf{X}}  

\newcommand{\g}{\mathbf{g}} 

\newcommand{\trace}{\operatorname{trace}}
\newcommand{\rank}{\operatorname{rank}}

\newcommand{\diag}{\operatorname{diag}}
\newcommand{\grad}{\operatorname{grad}}

\title{Dictionary-based Manifold Learning}

\author{%
  Hanyu Zhang \\
  Department of Statistics\\
 University of Washington\\
  Seattle, WA 98115 \\
  \texttt{hanyuz6@uw.edu} \\
   \And
  Samson Koelle\\
   Department of Statistics \\
  University of Washington\\
  Seattle,WA 98115\\
   \texttt{sjkoelle@gmail.com} \\
   \AND
  Marina Meil\u{a} \\
 University of Washington\\
  Seattle, WA 98115 \\
    \texttt{mmp@stat.washington.edu} \\
}

\begin{document}

\maketitle

\begin{abstract}
  We propose a paradigm for interpretable Manifold Learning for scientific data analysis, whereby we parametrize a manifold with $d$ smooth functions from a scientist-provided {\em dictionary} of meaningful, domain-related functions. When such a parametrization exists, we provide an algorithm for finding it based on sparse {\em non-linear} regression in the manifold tangent bundle, bypassing more standard manifold learning algorithms. We also discuss conditions for the existence of such parameterizations in function space and for successful recovery from finite samples. We demonstrate our method with experimental results from a real scientific domain.
\end{abstract}

\section{Introduction}
\label{sec:intro}

Dimension reduction algorithms map high-dimensional data into a
low-dimensional space by a learned function $f$.  However, it is often
difficult to ascribe an interpretable meaning to the learned
representation.  For example, in non-linear methods such as Laplacian
Eigenmaps \citep{belkin:01} and t-SNE \citep{tsne}, $f$ is learned
without construction of an explicit function in terms of the features.
In contrast, when scientists describe/model a system using knowledge from 
their domain, often the resulting model is in terms of domain relevant features, which are continuous functions of other domain variables (e.g. equations of motion). 

For example, in the application of Molecule Dynamic Simulation (MDS) study, data are often high dimensional with non-trivial topology, non i.i.d. noise. Figure \ref{fig:toluene_pca_features} shows pairwise scatterplots of six toluene molecule features and \ref{fig:toluene_diagram} displays a single scientifically relevant function that model (approximately) the state space of the toluene molecule; it is an angle of rotations.  

\begin{figure*}
    \centering
    \subfloat[]{\includegraphics[width=0.35\textwidth,clip]{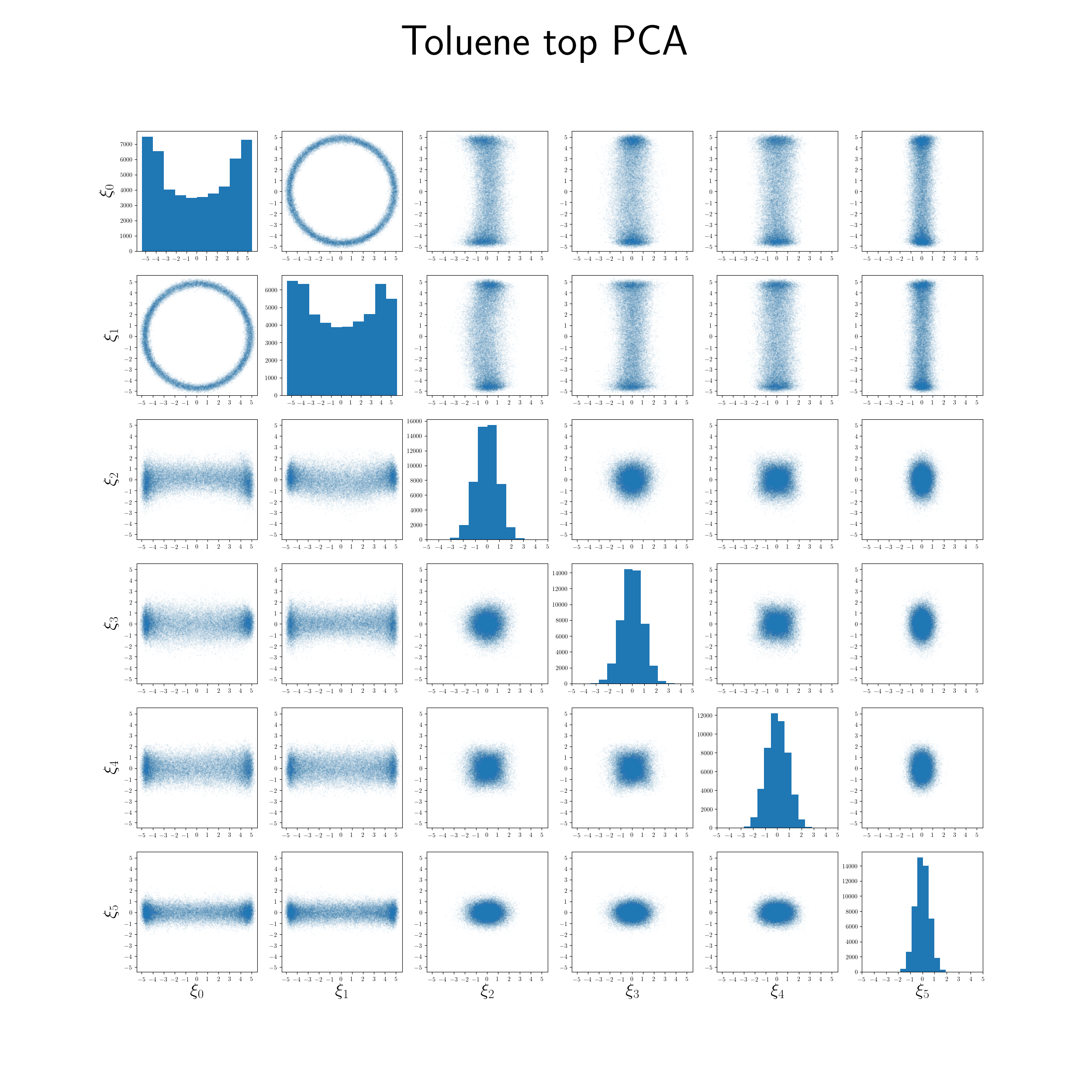}\label{fig:toluene_pca_features}}
    \subfloat[]{\includegraphics[width=0.35\textwidth,clip]{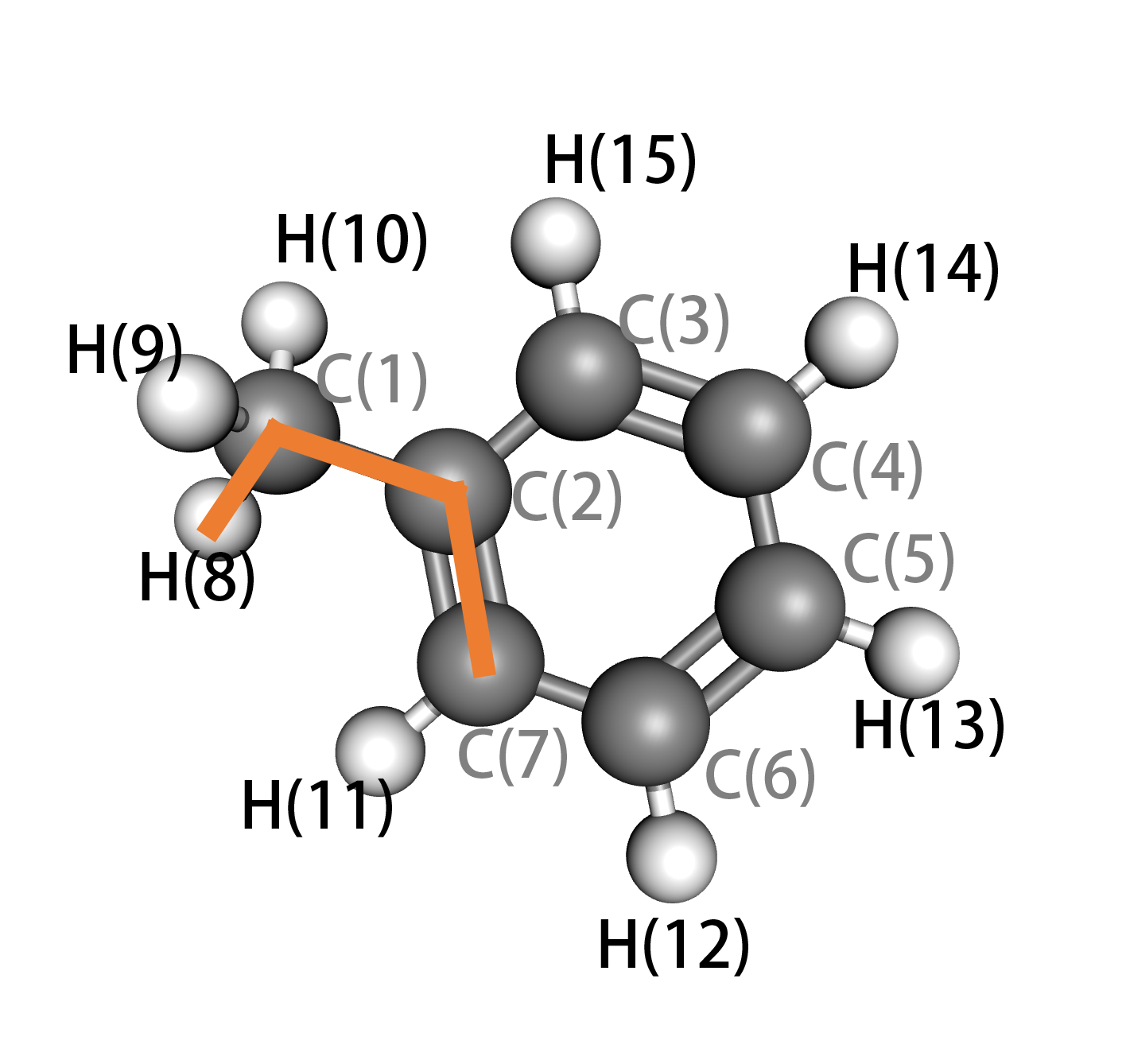}\label{fig:toluene_diagram}}
    \subfloat[]{\includegraphics[width=0.35\textwidth,clip]{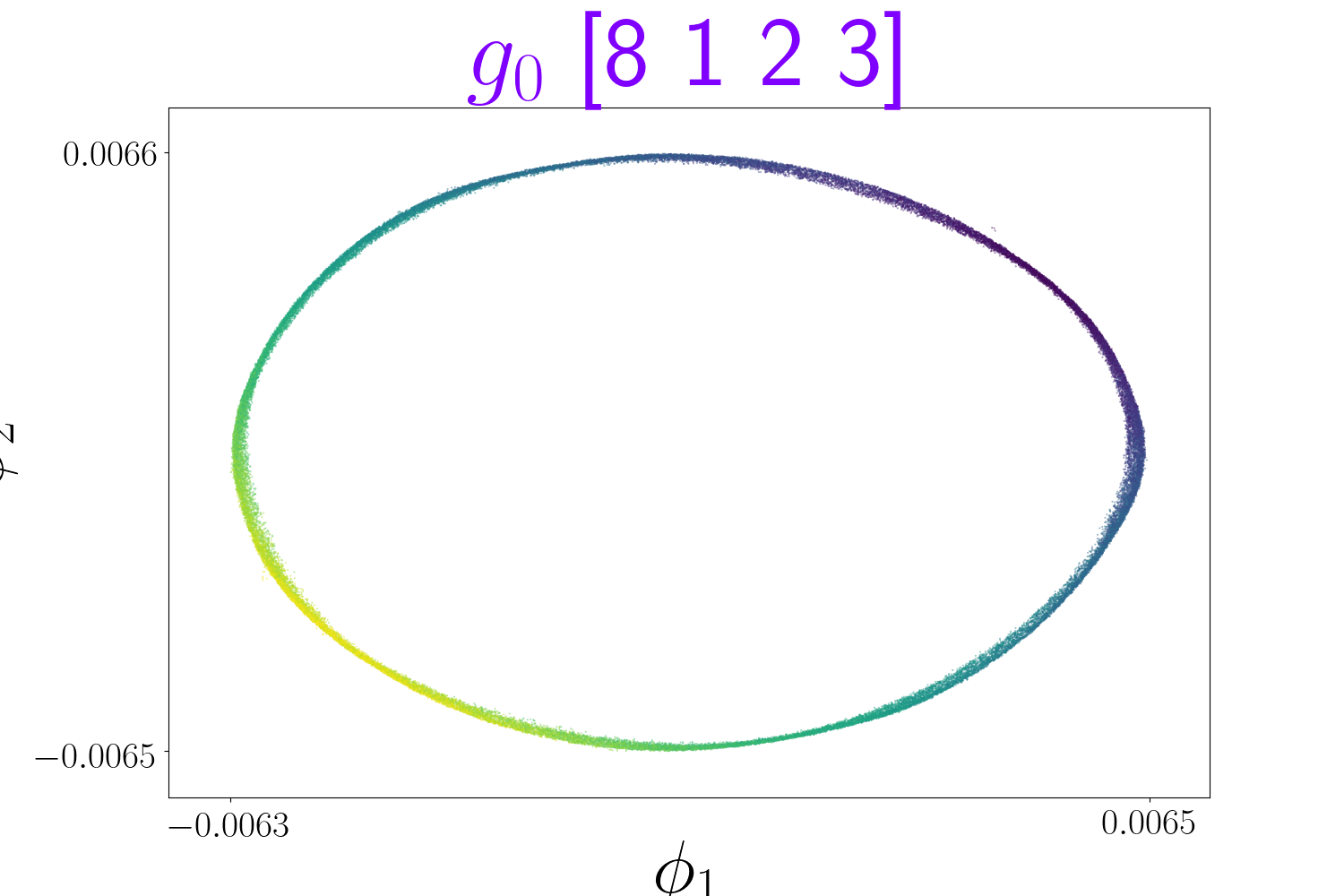}\label{fig:toluene_embedding}}\hfill
    \caption{Example of toluene molecule dynamic data. \textbf{Left:} pairwise scatterplots of first six coordinates in $\rrr^{50}$ and histograms of each coodinate on the diagonal. The preprocessing procedure is described in section \ref{sec:md}.  \textbf{Middle:} Atoms in a toluene molecule.  Scientists previously discovered that the torsion associated with the peripheral methyl group bound governs the state space of the toluene molecule as a one dimensional manifold. \textbf{Right:} Embedding of toluene data into $\rrr^2$ by diffusion map, colored by the bond torsion labeled. The variation of the color along the circle demonstrates this function as parametrizing the data manifold.}
\end{figure*}

A functional form $f$ can also be used to compare embeddings from different sources, derive out-of-sample extensions, and to interrogate mechanistic properties of the analyzed system. In figure \ref{fig:toluene_embedding}, we compare the scienfically identified functional mapping $f$ with existing manifold learning algorithms.

This paper proposes to construct a manifold model that interpolates between the two above modalities. Specifically, our algorithm will map samples $\xi_i$ from a manifold to new coordinates $f(\xi_i)$ like in purely data driven manifold learning, but these will be selected from a predefined {\em finite} set of smooth functions $\F$, called a {\em dictionary}, to represent intrinsic manifold coordinates of the data manifold $\M$.
Thus, the obtained embedding is smooth, has closed-form expression, can map new points from the manifold $\M$ to $f(\M)$ exactly, and is interpretable with respect to the dictionary.

This method, which we call \tsalg, requires the key assumption that the manifold $\M$ is parametrized by a subset of functions in the dictionary.
However, creating dictionaries of meaningful concepts for a scientific domain and finding those elements that well-describe the data manifold is an everyday task in scientific research.
We put the subset-selection task on a formal mathematical basis, and exhibit in Section \ref{sec:md} a scientific domain where the assumptions we make hold, and where our method replaces dictionary-based visual inspection of the data manifold.

\paragraph{Problem Statement}

Suppose data $\mathcal{D}=\{\xi_i,i\in[n]\}$ are sampled from a $d$-dimensional connected smooth\footnote{In this paper, by {\em smooth} manifold or function we mean of class $C^l$, $l\geq1$, to be defined in Section \ref{sec:theory}.} submanifold $\M$ embedded in the Euclidean space $\rrr^D$, where typically $D \gg d$. Assume that the intrinsic dimension $d$ is known. $\M$ has the Riemannian metric induced from $\rrr^D$. We are also given a dictionary of functions $\mathcal{F}=\{f_j,j\in[p]\}$. All of the functions $f_j$ are defined in the neighborhood of $\M$ in $\rrr^D$ and take values in some connected subset of $\rrr$. We require that they are smooth on $\M$ (as a subset of $\rrr^D$), and have analytically computable gradients in $\rrr^D$.  Our goal is to select $d$ functions in the dictionary, so that the mapping $f_S=(f_j)_{j\in S\subset \F}$ is a diffeomorphism on an open neighborhood $U\subset \M$ to $f_S(U)\subset \rrr^{|S|}$ at almost everywhere on $\M$, $f_S$ is then a \emph{global} mapping with fixed number of functions. The learned mapping $f_S$ will be a \emph{valid parametrization} of $\M$. 

The almost everywhere in the previous definition relaxes the usual definition of smooth embedding. Consider the circle embedded in $\mathbb
R^2$ by the map $g: t \mapsto (\cos t, \sin t)$ for $t\in \rrr$. Consider the function defined for $(x,y):|x^2+y^2-1|\leq 1/2$, then 
\begin{equation}
    \Theta: (x,y) \mapsto \begin{cases}\arcsin \frac{y}{\sqrt{x^2+y^2}} &x\geq 0 \\
    \pi-\arcsin \frac{y}{\sqrt{x^2+y^2}}, & x < 0
    \end{cases}
\end{equation}
is a valid parametrization for $\M$.

We had made two adjustments to standard differential geometry \citep{smoothmfd}.
First, in differential geometry terminology, $(U\subseteq\M,f_S)$ locally is a coordinate \emph{chart} for $\M$ and $f_S^{-1}$ is called a {\em parameterization} of $U$. In this paper, we often refer to $f_S$ as the 'parameterization', as $f_S,f_S^{-1}$ are diffeomorphisms and are both representative. We argue that $f_S$ is of more immediate interest, since this map consists of interpretable and analytically computable dictionary functions, and $f_S^{-1}$, while guaranteed to exist on $f_S(U)$, is defined only implicitly in many scenarios. 

Second, since a manifold may require multiple charts, we
relax the requirement that $f_S$ is locally a diffeomorphism
everywhere to {\em almost everywhere}. In the circle example, since the manifold $\M$ is compact, it is not possible to find a single smooth function that can locally be a diffeomorphism everywhere. This relaxation allows us to find $d$ functions parametrizing a $d-$dimensional compact manifold in our definition.

Our main technique is to operate over gradient fields on $\M$, which extends \citet{arxivVersion}.
In Section \ref{sec:bg}, we introduce some
backgrounds on gradient fields on manifolds.
In Section \ref{sec:tslasso}, we present our algorithm {\tsalg} in detail.
In Section \ref{sec:theory}, we provide sufficient conditions for selection consistency.
Section \ref{sec:md} shows experimental results on simulations and molecular dynamics
datasets.
Section \ref{sec:related} discusses related work and interesting features of our approach.

\section{Preliminaries: Gradients on Manifolds}

\label{sec:bg}
 The reader is referred to \citet{smoothmfd} for more backgrounds on differential geometry. In this section, we review gradient fields on manifolds, which play a central role in our algorithm. Consider a $d-$dimensional manifold $\M$. At point $\xi$, its tangent space $\T_{\xi}\M$ can be viewed as the equivalent class of directions of infinitesimal curves passing $\xi$. For a smooth function $f:\M\mapsto \rrr$, its differential $Df:\T_{\xi}\M \mapsto \rrr$ is a linear map that generalizes directional derivatives in calculus in Euclidean space, characterizing how the value of $f$ varies along different directions in $\T_{\xi}\M$. The chain rule also holds for compositions of functions on manifolds. 

 When $\M$ is Riemannian with metric $\g$, the gradient is a collection of tangent vectors $X(\xi)$, one at each point $\xi$, such that for all $\xi\in \M$ and all $v\in\T_{\xi}\M$
 \begin{equation}
     \langle X(\xi),v\rangle_{\g} = Df(v)\lvert_{\xi}\;.
     \label{eq:grad}
 \end{equation}

 For example, under the usual Euclidean metric, a function $f:\rrr^D\mapsto \rrr$ has a gradient vector $\nabla f(\xi)$ at each point $\xi\in\rrr^D$ as defined in ordinary multivariate calculus.
 
 For our problem, $\M$ is a $d-$dimensional manifold embedded in $\rrr^D$ with inherited metric. $\T_{\xi}\M$ can be identified as a $d-$dimensional linear subspace of $\T_{\xi}\rrr^D$, whose basis can be represented by an orthogonal $D\times d$ matrix $\mathbf{T}_{\xi}$. Let $f$ be a smooth real-valued function, defined on a open neighborhood of $\M$. There are two points of views for $f$ when it is restricted on $\M$: (i) as a function on $\rrr^D$ and has gradient $\nabla f$ as usual. (ii) as a function on $\M$ and one can show that the gradient field $\grad f$ given by the coordinate representation $\grad f := \mathbf{T}_{\xi}^\top \nabla f$ satisfies \eqref{eq:grad} \citep{smoothmfd}.

More generally, consider a map $F=(f_1,\cdots,f_s):\M\mapsto \rrr^s$. The differential $DF=(Df_1,\cdots,Df_s)$ is then defined to be a linear mapping from $\T_{\xi}\M\mapsto \T_{\xi}\rrr^s$. Under basis $\mathbf{T}_{\xi}$, a coordinate representation of $DF$ is $\mathbf{T}_{\xi}^\top \nabla F$, where $\nabla F$ is a $D\times s$ matrix, constructed buy row-wise stacking the gradients $\nabla f_1,\cdots,\nabla f_s$.


\section{The {\tsalg} algorithm}\label{sec:tslasso}

The idea of the {\tsalg} algorithm is to express the orthonormal bases $\mathbf{T}_{\xi}\in\mathbb{R}^{D\times d}$ of the manifold tangent spaces $\T_{\xi}\M$ as sparse linear combinations of dictionary function gradient vector fields. This simplifies the non-linear problem of selecting a best functional approximation to $\M$ to the linear problem of selecting best local approximations in the tangent bundle. If the subset $S$ with $|S|=d$ gives a valid parametrization, in a neighborhood $U_{\xi}\subset \M$ of almost all point $\xi$, $f_S$ is a diffeomorphism, i.e. there is some mapping $g:f_S(U_{\xi}) \mapsto U_{\xi}$ such that the identity map $f_S \circ g$ is identity map on $f_S(U_{\xi})$ and $g\circ f_S$ is the identity map on $U_{\xi}$. Thus, in coordinate representation we can denote a matrix representation of $Df_S(\xi)$ by $\mathbf{X}_{\xi,S}=\mathbf{T}_{\xi}^\top\nabla f_S(\xi) \in \rrr^{d\times d}$, and further there is some matrix $\mathbf{B}_{\xi,S}\in\rrr^{ d \times d}$ such that for all $\xi\in\M$
\begin{equation}
\mathbf{I}_d = \mathbf{X}_{\xi,S}\mathbf{B}_{\xi,S}
\label{eq:represent}
\end{equation}
according to the chain rule of function composition on manifolds.

For notation simplicity, we will write $\mathbf{X}_{iS},\mathbf{B}_{iS},\T_i \M$ as the corresponding quantities at point $\xi_i$ when we are discussing finite sample.
We can select $S=[p]$, and simplify the notation of $\mathbf{X}_{iS},\mathbf{B}_{iS}$ to $\mathbf{X}_i\in\mathbb{R}^{d\times p},\mathbf{B}_i\in\mathbb{R}^{p\times d}$, but crucially, if we do not have colinear gradients, then we can restrict all but $d$ rows of $\mathbf{B}_i$ to be zeros. We can also select $s=\{j\}$, and define $\mathbf{B}_{.j}\in \rrr^{nd}$ as the vector formed by concatenating $\mathbf{B}_{i\{j\}}$. Stacking $\mathbf{B}_{.j}$ together forms $\mathbf{B}\in\rrr^{p\times nd}$.

\subsection{Loss Function}
  We now seek a subset $S \subset [p]$ such that (1) only the corresponding $n d$ vectors $\mathbf{B}_{.j} : j \in S$ have non-zero entries and (2) each submatrix $\xb_{iS}$ forms a $\rank d$ matrix. The previous observation inspires minimizing Frobenius norm $\mathbf{I}_d-\mathbf{X}_{i}\mathbf{B}_{i}$ with joint sparsity constraints over rows of $\mathbf{B}_i$. This sparsity is also induced jointly over all data points. 

\beq
J_{\lambda_n}(\mathbf{B}) = \frac{1}{2} \sum_{i=1}^n \lvert\lvert \mathbf{I}_d-\mathbf{X}_{i}\mathbf{B}_{i}\lvert\lvert^2_F+\frac{{\lambda_n}}{\sqrt{dn}}\sum_{j=1}^p \lvert\lvert{\mathbf{B}_{.j}}\lvert\lvert_2.
\label{eq:obj}
\eeq
Note that this optimization problem is a variant of Group Lasso \citep{Yuan2006-af} that forces group of coefficients of size $dn$ to be zero simultaneously in the regularization path. The details of the tangent space estimation are deferred to Section \ref{sec:tangent}. It can be shown this loss function is invariant to local tangent space rotation.
%


\subsection{Tangent Space Estimation}
\label{sec:tangent}
So far we have solved our problem assuming we have access to the tangent space at each point $\xi\in\M$. However, this is rarely true. In practical use, the first step to realize the previous idea of expressing tangent spaces is to estimate them. \emph{Weighted Local Principal Component  Analysis} (WL-PCA) algorithm proposed as \citet{Singer2011VectorDM,Chen2013,aamari2018} are exmaples to estimate such basis. These methods are shown to have accurate tangent space estimation when the hyperparameters are selected appropriately. 

Intuitively, estimating tangent spaces is estimating local covariances matrices centered at each point $\xi_i$. We therefore select a neighborhood radius parameter $r_N$ and identify $\mathcal{N}_i=\{i'\in[n], \text{with \ } \norm{\xi_i-\xi_{i'}}_2 \leq r_N\}$ to be all neighbor points of $\xi_i$ within Euclidean (in $\rrr^D$) distance $r_N$  so that we can pass  into this algorithm.

When compute local covariance matrices, one may weight different points. These weights of each $\xi_j$ in $\neigh_i$ can be chosen to be proportional some kernel function $K(x)$ such that for all $j\in\neigh_i$ the weight is proportional to $K_{ij}= K(\norm{\xi_i-\xi_{j}}/\epsilon_N)$, where $\epsilon_N$ is a tuning-parameter proportional to $r_N$ in the sense that kernel-values of pairs of non-neighboring points should be close to zero. Any $C^2$ positive monotonic decreasing function $K(u)$ with compact support is valid; examples including constant kernel $K(u)=1_{[0,1]}(u)$, Epanechnikov $K(u)=(1-u^2)1_{[0,1](x)}$ and Gaussian $K(u)=\exp(-u^2)1_{[0,1](x)}$ etc. We specifically choose the Gaussian kernel in our experiments since it provides better tangent space estimation empirically, as it weights more on points that are close to where the tangent space is of interest. Given these weights $K_{ij}$ for $\xi_j$s, the local weighted mean and weighted covariance at $\xi_i$ can be estimated, and singular value decomposition is used to find the basis.

Let $k_i = |\mathcal{N}_i|$ be the number of neighbors of point $\xi_i$ and $\mathbf \Xi_i = \{\xi_{i'},i'\in\mathcal{N}_i\}\in\rrr^{|\mathcal{N}_i|\times D}$ be the correpsonding local position matrices. Also denote a column vector of ones of length $k$ by $ \bm{1}_k$, and define the Singular Value Decomposition algorithm $\text{SVD}(\mathbf X,d)$ of matrix $\mathbf X$ as outputting $\mathbf V, \Lambda$, where $\Lambda$ and $\mathbf V$ are the largest $d$ eigenvalues and their corresponding eigenvectors. Tangent space estimation algorithm is displayed in algorithm \tppalg~.
\begin{algorithm}
\caption{\tppalg} 
\begin{algorithmic}[1]
\STATE {\bfseries Input:} Local dataset $\mathbf \Xi_i$,  intrinsic dimension $d$,  kernel parameter $\epsilon_N$
\STATE Compute local kernel weights $K_{i,\neigh_i}=(K_{ij})_{j\in\neigh_i} \in \rrr^{k_i} $.
\STATE Compute weighted mean $\bar \xi_i = (K_{i,\neigh_i}^\top \bm{1}_{k_i})^{-1} K_{i,\neigh_i}^\top \mathbf \Xi_i $
  \STATE Compute weighted local difference matrix $\mathbf Z_{i} =  \diag(K_{i , \neigh_i}^{\frac{1}{2}}) (\mathbf \Xi_i - \bm{1}_{k_i} \bar \xi_i)$
  \STATE Compute $\mathbf{T}_i, \Lambda \leftarrow \text{SVD} (\mathbf Z_i^\top  \mathbf Z_i, d)$   
  \STATE {\bfseries Output:} $\mathbf{T}_i$ \label{alg:tan}
\end{algorithmic}
\end{algorithm}

\subsection{The \tsalg~Algorithm}
\label{sec:whole-algo}
We now present the full {\tsalg} approach. Following the logic in \ref{sec:tslasso}, we transform our non-linear manifold parameterization support recovery problem into a collection of sparse linear problems in which we express coordinates of individual tangent spaces as linear combinations of gradients of functions from our dictionary.
Tangent spaces at each point are estimated in step \ref{alg:tan}, enabling utilizing gradients of dictionary functions in $\T_\xi \M$ by projecting the gradient $\nabla f_j(\xi_i)\in \rrr^D$ on to estimated tangent spaces $\mathbf {T}_i$.
Finally we input these gradients into objective function \eqref{eq:obj} to solve for the support. 
\begin{algorithm}
\caption{{\tsalg}}
\label{alg:tsalg}
\begin{algorithmic}[1]
\STATE{\bfseries Input:} Dataset $\dataset$, dictionary $\G$, intrinsic dimension $d$, regularization parameter $\lambda_n$, radius parameter $r_N$, kernel parameter $\epsilon_N$.
  \FOR {$i=1,2,\ldots n$ (or subset $I \subset [n]$) }
    \STATE Compute $\neigh_i$ and $\mathbf \Xi_i$ using $\dataset, r_N$
  \STATE Compute the orthonormal tangent space basis  $\mathbf{T}_i \gets $\tppalg$(\mathbf \Xi_i, d,\epsilon_N)$  \label{alg:tan}
  \STATE Compute $\nabla f_j (\xi_i)$ for $j\in[p]$ \label{alg:dict}.
  \STATE Project onto tangent space\\ $\xb_i = \mathbf{T}_i^\top[\nabla f_j(\xi)]_{j\in [p]}$
  \ENDFOR
\STATE Solve for $\mathbf{B}$ by minimizing  $J_{\lambda_n}(\mathbf{B})$ in \eqref{eq:obj}. 
\STATE {\bfseries Output:} $S=\{j\in[p]:\norm{\mathbf{B}_{.j}}_2 > 0 \}$  \label{alg:supp}
\end{algorithmic}
\end{algorithm}


\subsection{Other considerations}
\label{sec:other}

\paragraph{Normalization}

The rescaling of functions $f_j$ will affect the solution of the Group
Lasso objective, since functions with larger gradient norm will tend to
have smaller $\parallel \mathbf{B}_{.j}\parallel$. This can affect the support $S$
recovered. Therefore, we compute $\gamma_j^2 = \frac{1}{n}\sum_{i=1}^n
\lvert\lvert \nabla f_j(\xi_i)\lvert\lvert^2$ and set $f_j \leftarrow
f_j/\gamma_j$. This approximates normalization by $\lvert\lvert \nabla
f_j \lvert\lvert_{{L}_2(\mathcal{M})}$. Since $\lvert\nabla
f_j(\xi_i)\lvert^2=\lvert \grad f_j(\xi_i)\lvert^2+\lvert\nabla
f_j^\perp(\xi_i)\lvert^2$, where $\nabla f_j^\perp$ denotes the
component of $\nabla f_j$ orthogonal to $\M$, normalization prior to
projection penalizes functions with large $\nabla f_j^\perp$ and
favors functions whose gradients are more parallel to the tangent space of $\M$. Note that, in the high-dimensional setting, we expect random functions to have gradient perpindicular to $\T \M$, and so these will be penalized by our normalization strategy.

\paragraph{Computation}

Note that we do not need to run \tsalg~ on our whole dataset in order to take advantage of all of our data, and can instead run on a subset $I \subset [n]$ such that $|I| = n'$. In particular, the search task in identifying the local datasets $\mathbf \Xi_i$ is $O(Dn n')$, which is significantly less than the time to construct a full neighbor graph for an embedding. For each $i$, computing the local mean is $O(k_i D)$, and finding the tangent space is $O(k_i D^2 + k_i^3)$. Gradient computation runtime is $O(D)$, but the constant may be large. Projection is $O(dDp)$. For each Group Lasso iteration, the compute time is $O(n' m pd)$ \citep{arxivVersion}.

\paragraph{Tuning} 
For the real data experiments, we select $\epsilon_N$ using the method of \citet{Joncas2017-kn}, while in simulation, we set it proportional to noise. 
As explained in the next section, we are theoretically motivated by the definition of parameterization to select a support $S$ that has cardinality equal to $d$, which is assumed to be given, although dimension estimation as in \citet{LevinaB04} could also be appropriate.
For $\lambda$, we apply binary search to the regularization path from $\lambda = 0$ to $\lambda_{\text{max}} = \max_{j} ( \sum_{i=1}^n ( \| \grad_{T_i^\M} f_j (\xi_i))\|_2^2  )^{1/2}$ to find $\lambda$ s.t. the cardinality of the selected support is $d$.
In the next section, we introduce support recovery conditions for the success of this approach, and introduce a variation of \tsalg~ for when they are violated.




\section{Support Recovery Guarantee}
\label{sec:theory}
In this section, we discuss the behavior of {\tsalg~} theoretically. First, we discuss the existence and uniqueness of a group of functions $f_S\subset \mathcal{F}$ that can serve as a valid parametrization. When such minimal parametrization exists and is unique, we provide sufficient conditions so that {\tsalg~} correctly selects this group with high probability w.r.t. sampling on the manifold and this probability converges to one if sample size tends to infinity. 
\begin{assumption}\label{as:all} Throughout this section, we assume the followings to be true.
    \begin{enumerate}
        \item \label{as:M} $\M$ is a $d$-dimensional $C^\ell,\ell\geq 1$ compact manifold with reach $\tau>0$ embedded in $\mathbb{R}^D$ with inherited Euclidean metric.
        \item \label{as:D} Data $\{\xi_i\}_{i=1}^n$ are sampled from some probability measure $P$ on the manifold that has a Radon-Nikodym derivative $\pi(\xi)$ with respect to the Hausdorff measure.  There exist two positive constants $\pi_{\min},\pi_{\max}$ such that $0< \pi_{\min}\leq \pi(\xi) \leq \pi_{\max}$ for all $\xi\in\M$.
        \item \label{as:F} Dictionary $\F=\{f_j(\xi):j\in[p]\}$ contains $p$ $C^1$ functions defined on a neighborhood of $\M$ in $\mathbb{R}^D$. Further assume that $\delta:=\inf_{\xi\in\M}\min_{j\in[p]}\norm{\nabla f_j(\xi)} > 0$ and denote $\Gamma:=\sup_{\xi\in\M}\max_{j\in[p]} \norm{\nabla f_j(\xi_i)}$.
        \item \label{as:S} $S\subset[p], |S|=d$ is the only subset such that $\rank f_S = d$ a.e. on $\M$ w.r.t. Hausdorff measure.
    \end{enumerate}
\end{assumption}
Assumption \ref{as:M} on manifold and \ref{as:D} on sampling are common in the manifold estimation literature (e.g. \citet{aamari2018}). The positive reach in \ref{as:M} will avoid extreme curvature and bizarre behavior of the manifold, and the assumption \ref{as:D} on the density enforces the uniformity of sampling. Assumption \ref{as:F} restricts the smoothness of all dictionary functions and ensures that all dictionary functions do not have critical points on $\M$ as a function on $\rrr^D$. One should also notice that $\Gamma< \infty$ by the compactness assumption of $\M$. 

Now we are ready to prove support recovery consistency under suitable conditions. Let $\hat{\mathbf{B}}$ be the solution of problem \eqref{eq:obj} and $S(\hat{\mathbf{B}})$ be the nonzero rows of $\hat{\mathbf{B}}$. We will show that the probability of $S(\hat{\mathbf{B}})=S$ converges to 1 as $n$ increases. We start by defining
 \begin{equation}
   b_S = \inf_{\xi:\rank Df_S(\xi) = d}\min_{j\in S} \norm{ \mathbf{B}_{\xi,\{j\}} }_2
   \label{eq:b-global}
 \end{equation}
Larger $b_S$ is an indicator of higher strength of signal.
Further consider the matrix $\tilde{\xb}_{\xi}$ whose $j$-th column is $\xb_{\xi,\cdot j}/\norm{\nabla f_j(\xi)}$. Correspondingly we can define $\tilde{\xb}_{\xi,S}$ as the submatrix of $\tilde{\xb}$ with columns in $S$. Let $\mathbf{G}_{\xi,S}=\diag\{\norm{\nabla f_j(\xi)}\}_{j\in S}$ and define
\begin{align}
     \mu_S &= \sup_{\xi\in\M,j\in S,j'\notin S} |\tilde{\xb}_{\xi,\cdot j}^\top \tilde{\xb}_{\xi,\cdot j'}|\;,
    \label{eq:mu-global}\\
    \nu_S &= \sup_{\xi\in \M} \norm{(\tilde{\xb}_{\xi,S}^\top\tilde{\xb}_{\xi,S})^{-1}-\mathbf{G}_{\xi,S}^{2}}\;.
    \label{eq:nu-global}
\end{align}
Here $\nu_S$ is finite if $\mu_S < 1/(d-1)$, guaranteed by the Gershgorin circle theorem.
The parameter $\mu_S$ can be thought of as a renormalized incoherence between the functions in $S$ and those not in $S$; 
$\nu_S$ is a internal colinearity parameter, which is small when the columns of $\xb_{S}(\xi)$ are closer to orthogonality and the gradient of functions in $S$ are more parallel to the tangent space. We also define
\begin{align}
    \phi_S = \sup_{\xi\in\M}\max_{j\in S}\norm{\nabla f_j(\xi)}_2
    \label{eq:phi-global}
\end{align}
which upper bounds the Euclidean gradient of functions in $S$.

\begin{proposition}
Suppose Assumptions \ref{as:all} hold. In algorithm \ref{alg:tsalg}, suppose tangent spaces are estimated by WL-PCA in Section \ref{sec:tangent} using Gaussian kernel and bandwidth parameter choice $\epsilon_N=r_N=C(({\log n}/{(n-1)})^{1/d})$ with large enough constant $C$, and normalization on dictionary is performed as in Section \ref{sec:other}.  If  $(1+{\nu_S}/{\delta^2})^2\mu_S\phi_S\Gamma d<1$ and  $\lambda_n(1+{\nu_S}/{\delta^2})^2< b_S\sqrt{n}/2$, then there is a constant $N$ depending only on $\M, \pi_{\min},\pi_{\max}$ such that when $n>N$, it holds that
 \begin{equation}
     Pr(S(\widehat{\mathbf{B}}) = S) \geq 1-4(\frac{1}{n})^{\frac{2}{d}}
 \end{equation}
 \label{prop:recovery}
\end{proposition}
The proof is contained in the supplementary material. The main idea is first to find a sufficient condition so that given correct gradient of each function {\tsalg} can find the correct support, assuming correct estimation of the tangent space. Then we consider this condition in the case where gradient is estimated from data and obtain the guarantee by the fact that tangent spaces can be consistently estimated with larger sample size. 

There are some differences to be noted of this recovery result compared with classical recovery guarantees in Group Lasso type problems in e.g. \citet{Wainwright:2009sharp}, \citet{GO2011}, \citet{Elyaderani2017-ce}.  
First, we cannot adopt directly the usual assumption in Lasso literature that each column of $\xb$ has unit norm, considering the normalization in Section \ref{sec:tangent}. 
Also, the asymptotic regime we are considering here is only $n\rightarrow \infty$. Although we are using a Group Lasso type optimization problem, the dimension $p$ is fixed since we only consider the fixed dictionary. There is no other conditions between $p$ and $n$ in our result, as required in many literature.
Third, the noise structure is not the same as a general Group Lasso problem since the source of noise is estimation of tangent space. Since we are sampling \emph{on the manifold}, there is no noise level parameter that appears in standard Lasso literature. In a simulation experiment, we also explore the behavior of our method on noisy settings.

\section{Experiments}
We illustrate the behavior of \tsalg~ on both synthetic and real data. Our synthetic data sets include a swiss roll in $\rrr^{49}$ and a rigid ethanol data in $\rrr^{50}$ and our real datasets are data molecular dynamics simulation (MDS) for three different molecules (Ethanol , Malonaldehyde and Toluene). Due to space limit we only present result of real datasets here. Results on synthetic datasets are included in the supplementary materials. 

For all of the experiments, the data consist of $n$ data points in $D$ dimensions.
\tsalg~ is applied to a uniformly random subset of size $n' = |\I|$ using $p$ dictionary functions, and this process is repeated $\omega$ number of times.
Note that the entire data set is used for tangent space estimation.
In our experiments, the intrinsic dimension $d$ is assumed known, but could be estimated by a method such as in \citet{LevinaB04}.
The local tangent space kernel bandwidth $\epsilon_N$ is estimated using the algorithm of \citet{Joncas2017-kn} for molecular dynamics data.
Parameters are summarized in Table \ref{tab:exps}.Experiments were performed in Python on a 16 core Linux Debian Cluster with 768 gigabytes of RAM. Code is available at \url{github.com/codanonymous/tslasso}.  Data is available at https://\url{figshare.com/s/fbd95c10b09f1140389d}.
\begin{table}[!h]
    \centering
    \begin{tabular}{l|c|c|c|c|c|c|c|c}
    \hline
    \hline
    Dataset       & $n$   & $N_a$ & $D$ & $d$ & $\epsilon_N$ & $n'$ & $p$ & $\omega$ \\
    \hline
    Eth       & 50000 & 9     & 50  & 2   & 3.5          & 100  & 12  & 25       \\
    Mal & 50000 & 9     & 50  & 2   & 3.5          & 100  & 12  & 25       \\
    Tol       & 50000 & 15    & 50  & 1   & 1.9          & 100  & 30  & 25     \\ 
    \hline
    \hline
    \end{tabular}
    \caption{Parameters in different experiments: Eth (Ethanol), Mal (Malonaldehyde) and Tol (Toluene)}
    \label{tab:exps}
\end{table}

These simulations dynamically generate atomic configurations which, due to interatomic interactions, exhibit non-linear, multiscale, non-i.i.d. noise, as well as non-trivial topology and geometry.
That is, they lie near a low-dimensional manifold \cite{dasMSKClementi:06}. Such simulations are reasonable application for \tsalg~ because there is no sparse parameterization of the data manifold known a priori.
Such parameterizations are useful.
They provide scientific insight about the data generating mechanism, and can be used to bias future simulations.
However, these parameterizations are typically are detected by a trained human expert manually inspecting embedded data manifolds for covariates of interest.
Therefore, we instead apply \tsalg~ to identify functional covariates that parameterize this manifold.

\paragraph{Experiment Setups}

We obtain a Euclidean group-invariant featurization of the atomic coordinates as a vector of planar angles $a_i \in\rrr^{3 {N_a \choose 3}}$: the planar angles formed by triplets of atoms in the molecule
\citep{ChenMcqueenKoelleMChmielaTkatchenko:mlcules-dum19}.
We then perform an SVD on this featurization, and project the data onto the top $D = 50$ singular vectors to remove linear redundancies.
Note that this represents a particular metric on the molecular {\em shape space}.

The dictionaries we considered are constructed on \emph{bond diagram}, a priori information about molecular structure garnered from historical work. Building a dictionary based on this structure is akin to many other methods in the field \citep{Krenn2020-aj, Xie2019-kw}.  Specifically, this dictionary consist of all equivalence classes of 4-tuples of atoms implicitly defined along the molecule skeletons. 

Since original angular data featurization is an overparameterization of the shape space, one cannot use automatically obtained gradients in \tsalg. We therefore project the gradients prior to normalization on the tangent bundle of the shape space as it is embedded in $\rrr^D$.

For \tsalg, the regularization parameter $\lambda_n$ ranges from 0 to the value for which $||\mathbf B_{.j} ||_2 =0$ for all $j$. The last $d$ surviving dictionary functions are chosen as the parameterization for the manifold.

\paragraph{Results on MDS Data}\label{sec:md}

The toeuene case is a manifold with $d=1$. We observe that in all replicates, \tsalg successfully select one of the six torsions associated with the peripheral methyl group bond, which shows the ability of our algorithm to automatically select appropriate parametrizing functions.

We plot the incoherence for Ethanol and Malonaldehyde as the heatmap in figure \ref{fig:ethanol-cosine-colored} and \ref{fig:malonaldehyde-cosine-colored}, which present two groups of highly linearly dependent torsions, corresponding to the two bonds between heavy atoms in the molecules. Therefore, we expect to select a pair of incoherent torsions out of these dictionaries. In figure \ref{fig:malonaldehyde-watch} and \ref{fig:ethanol-watch}, support recovery frequencies for sets of size $d = s =2$ using \tsalg~ on ethanol and malonaldehyde data respectively. As we expected, \tsalg~ select one function from the two groups of highly colinear functions in most replicates. These results shows that our approach is able to identify embedding coordinates that are comparable or preferable to the a priori known functional support.

Results such as these usually are generated subsequent to running a non-parametric manifold learning algorithm, either through visual or saliency-based analyses, but we are able to achieve comparable results without the use of such an algorithm. See supplementary materials for a comparison of our algorithm with other manifold learning algorithms. These results also suggest that the local denoising property of the tangent space estimation, coupled with the global regularity imposed by the assumption that the manifold is parameterized by the same functions throughout, is sufficient to replicate the denoising effect of a manifold learning algorithm. Plus, with the help of domain functions, our embeddings come with good interpretibility.

Also we point out that in our experiments, the subsampled size $n'=100$ is only around 1\% of the whole dataset and in almost all replicates this subsample is sufficient to obtain a valid parametrization. Tangent space estimation is only needed for these points. Therefore bypassing the usual manifold embedding procedure (on the whole dataset) we are able to obtain interpretable embeddings with fewer samples and in a shorter time.

\begin{figure*}
    \centering
\subfloat[]{\includegraphics[width=0.22\textwidth]{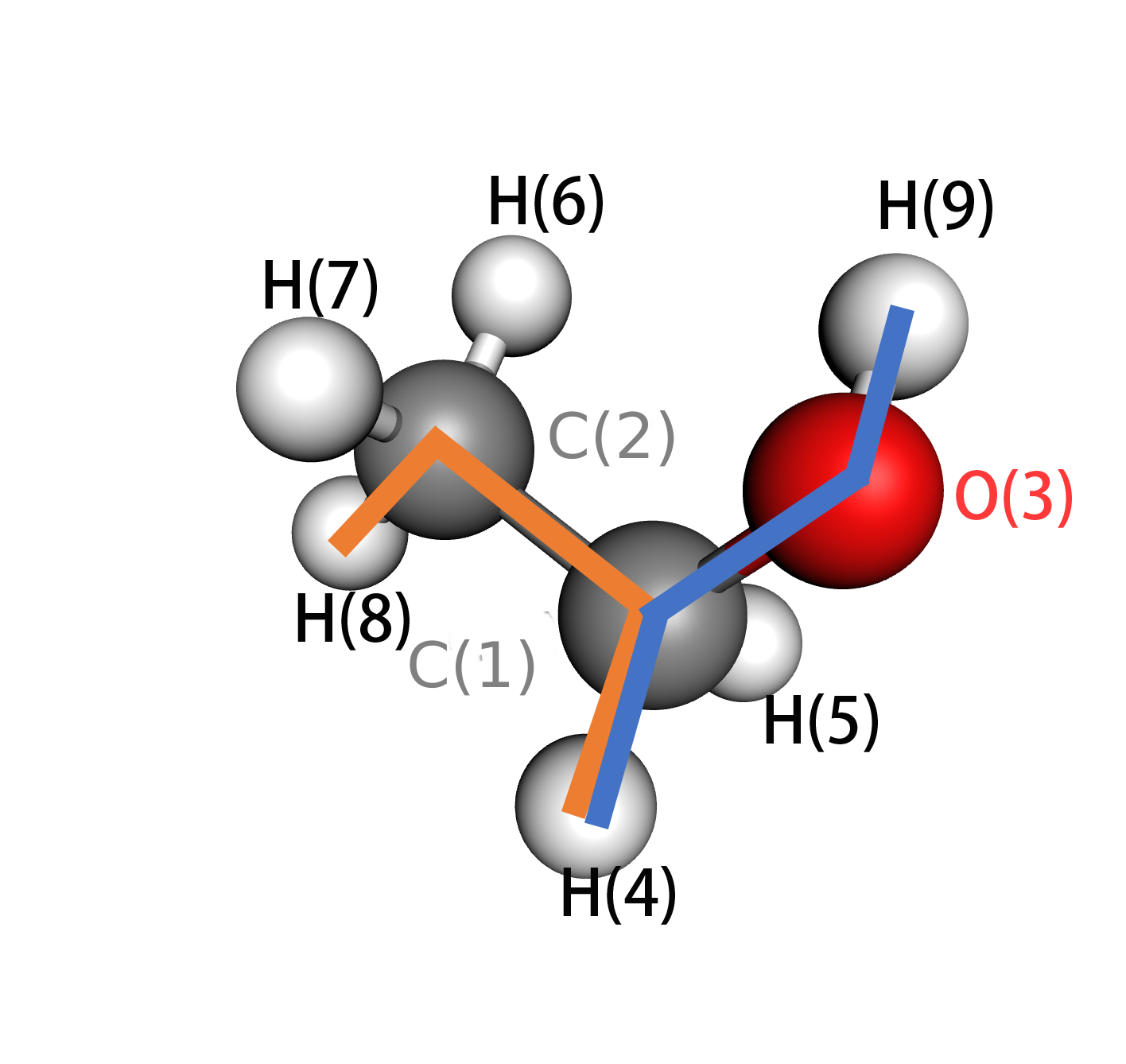}\label{fig:ethanol-diagram}}\hfill
\subfloat[]{\includegraphics[width=0.22\textwidth]{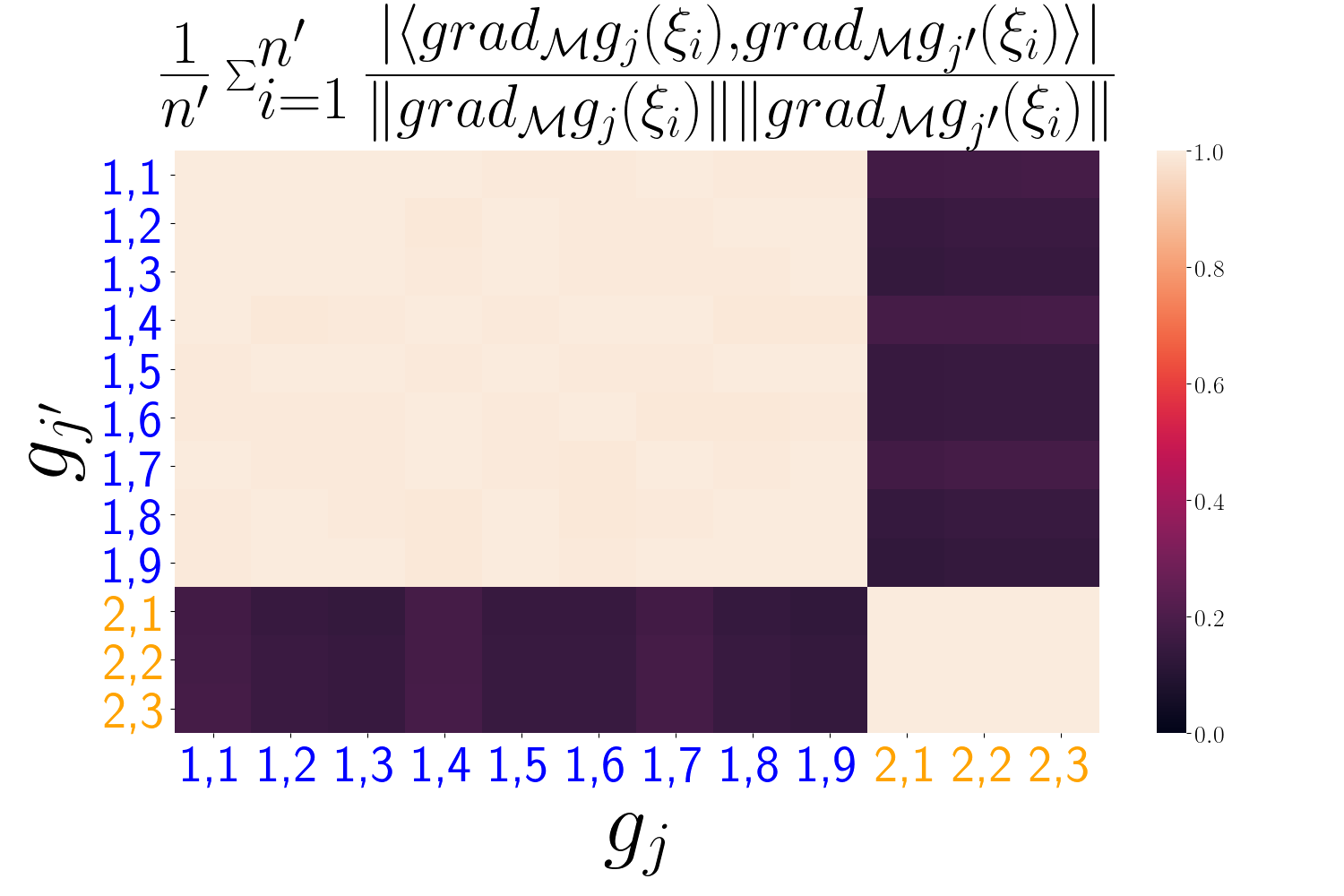}\label{fig:ethanol-cosine-colored}}\hfill
\subfloat[]{\includegraphics[width=0.22\textwidth]{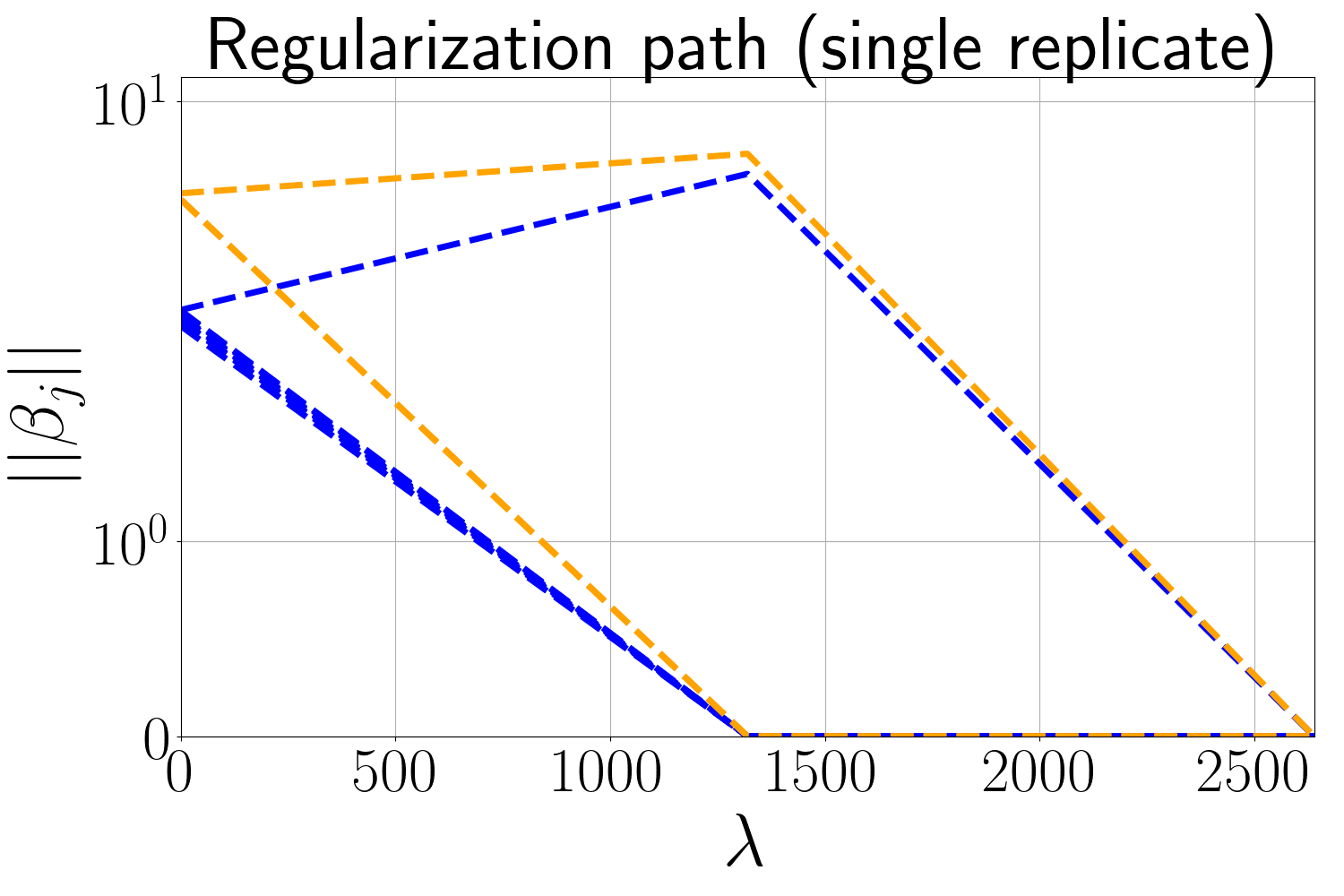}\label{fig:ethanol-regularizationpath}}\hfill
\subfloat[]{\includegraphics[width=0.22\textwidth]{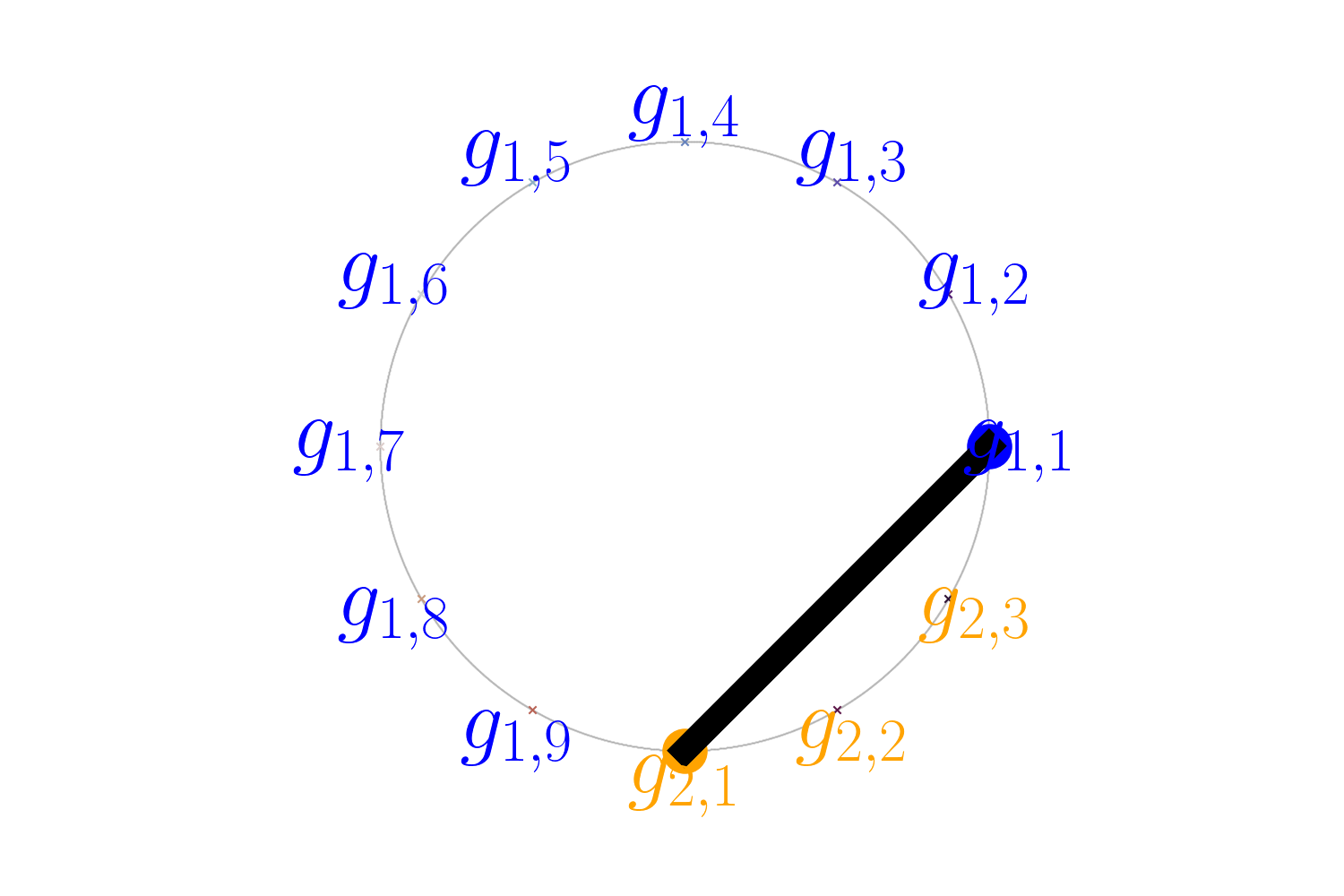}\label{fig:ethanol-watch}}\hfill
\newline
\subfloat[]{\includegraphics[width=0.22\textwidth]{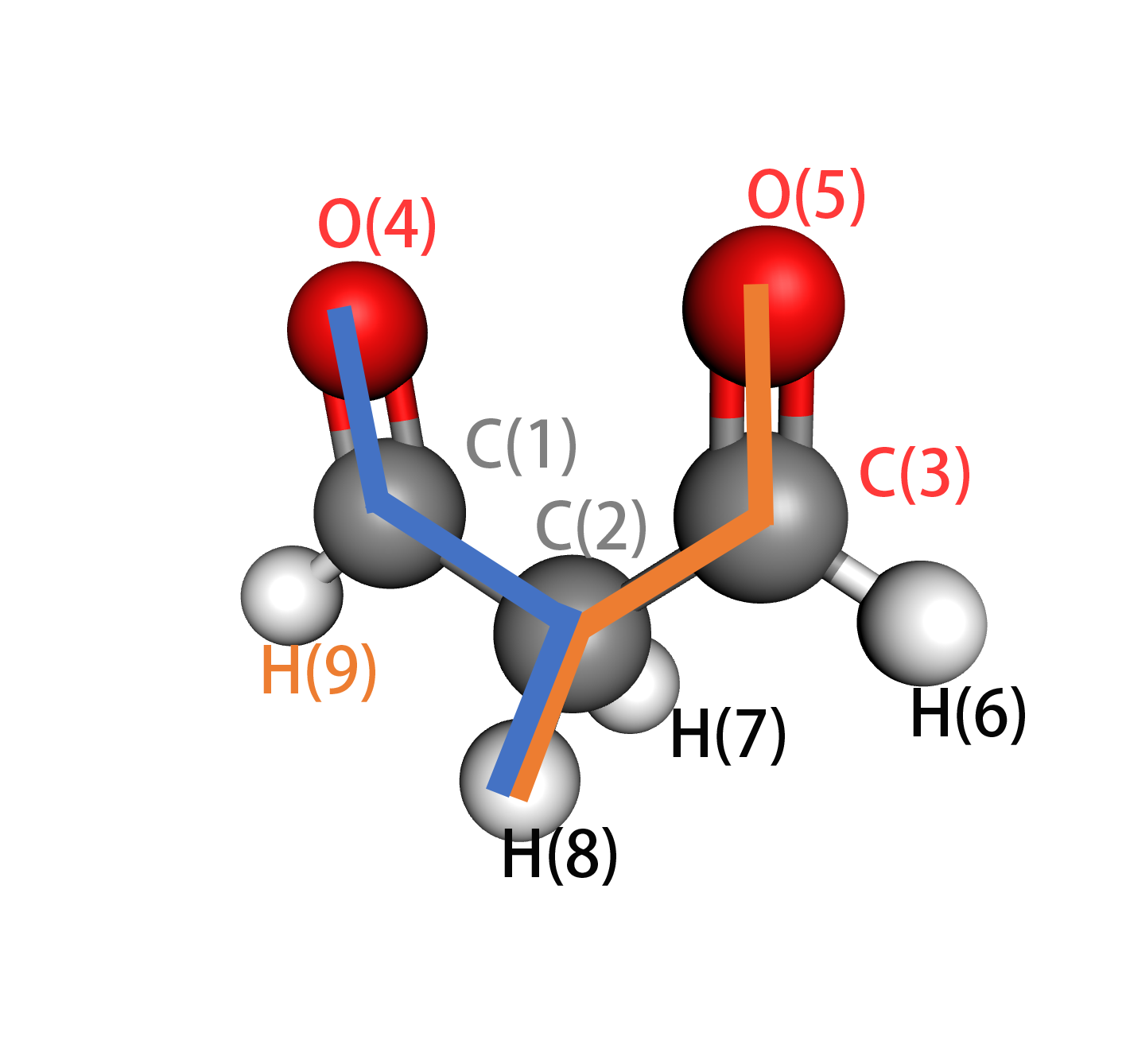}\label{fig:malonaldehyde-diagram}}\hfill
\subfloat[]{\includegraphics[width=0.22\textwidth]{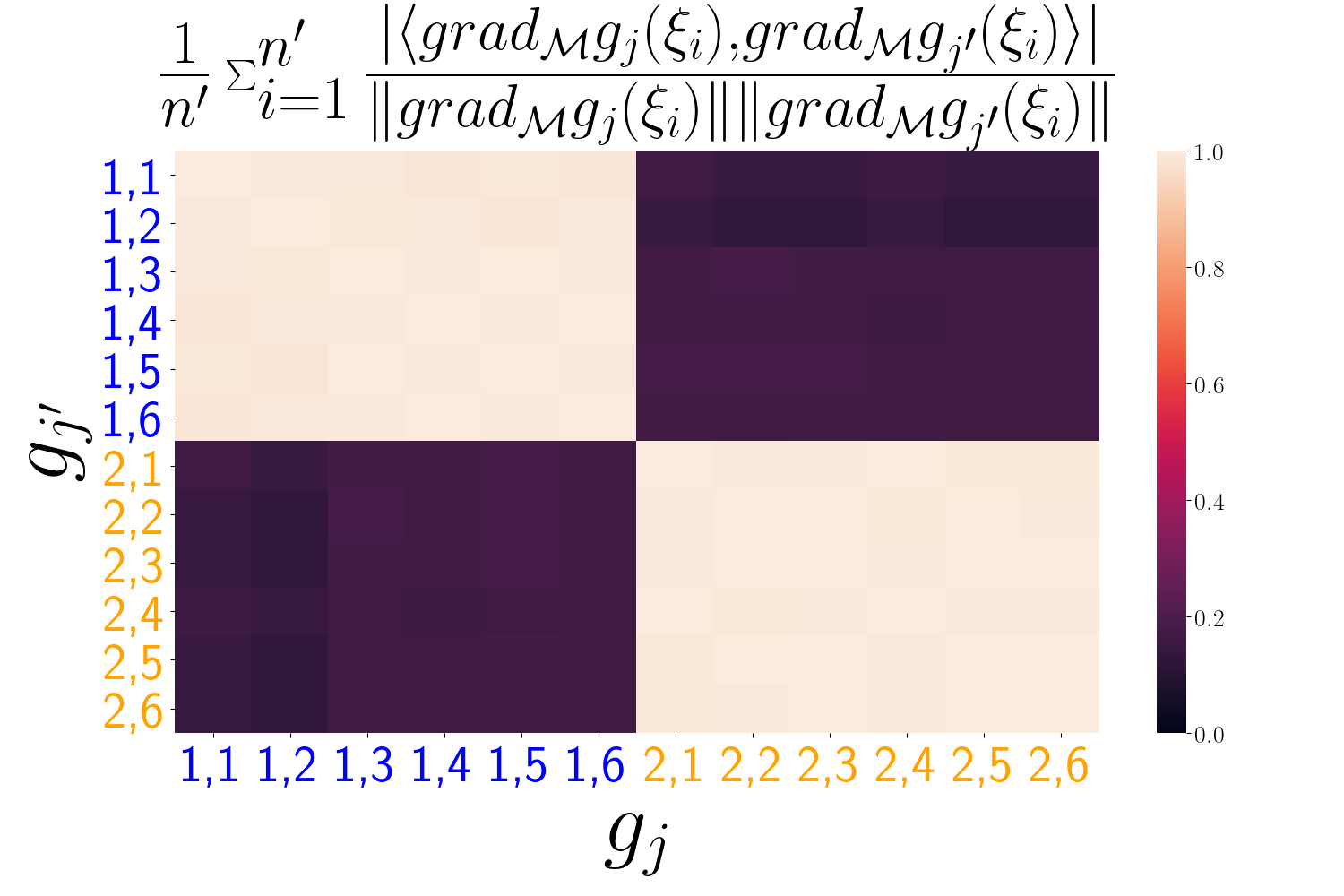}\label{fig:malonaldehyde-cosine-colored}}\hfill
\subfloat[]{\includegraphics[width=0.22\textwidth]{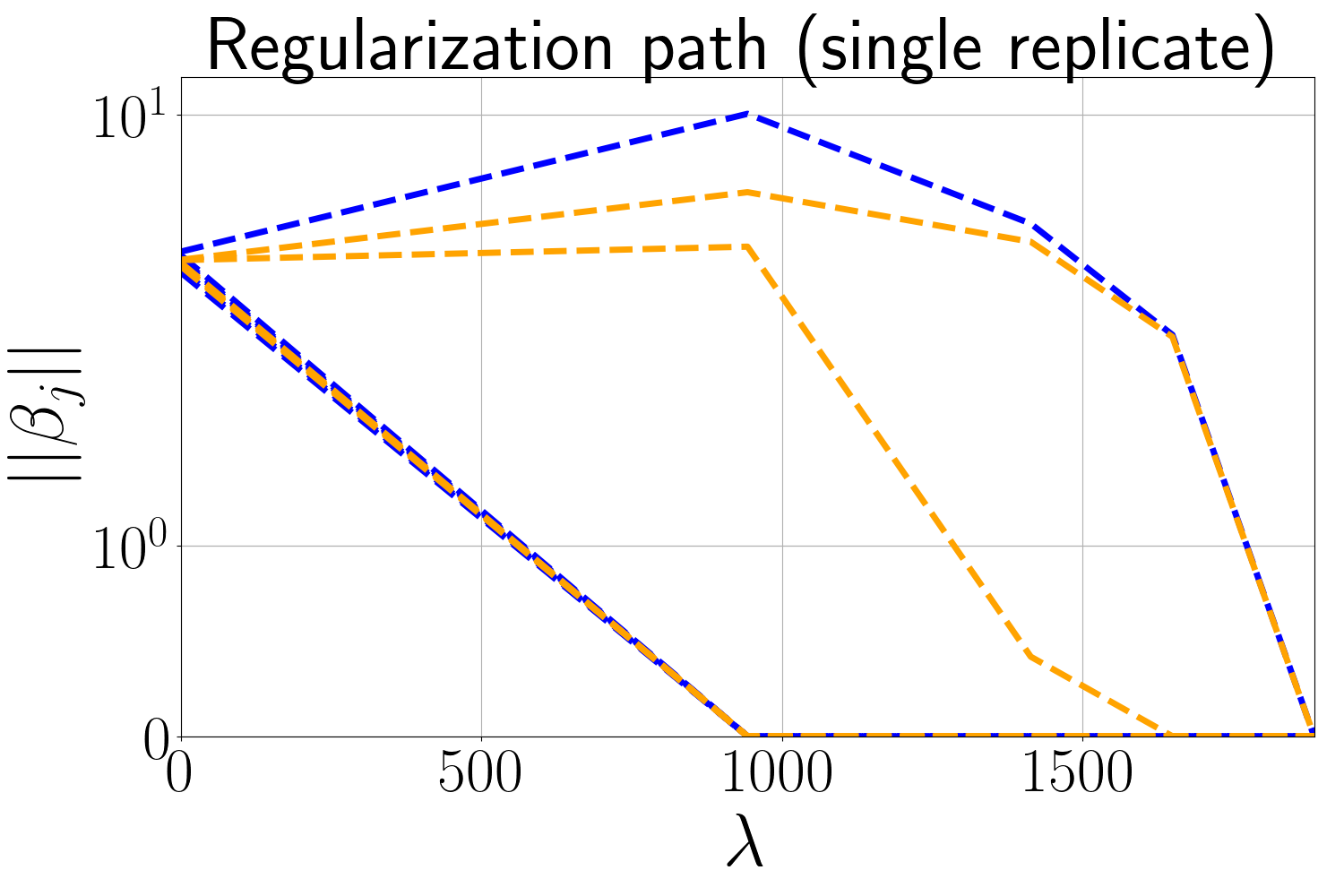}\label{fig:malonaldehyde-regularizationpath}}\hfill
\subfloat[]{\includegraphics[width=0.22\textwidth]{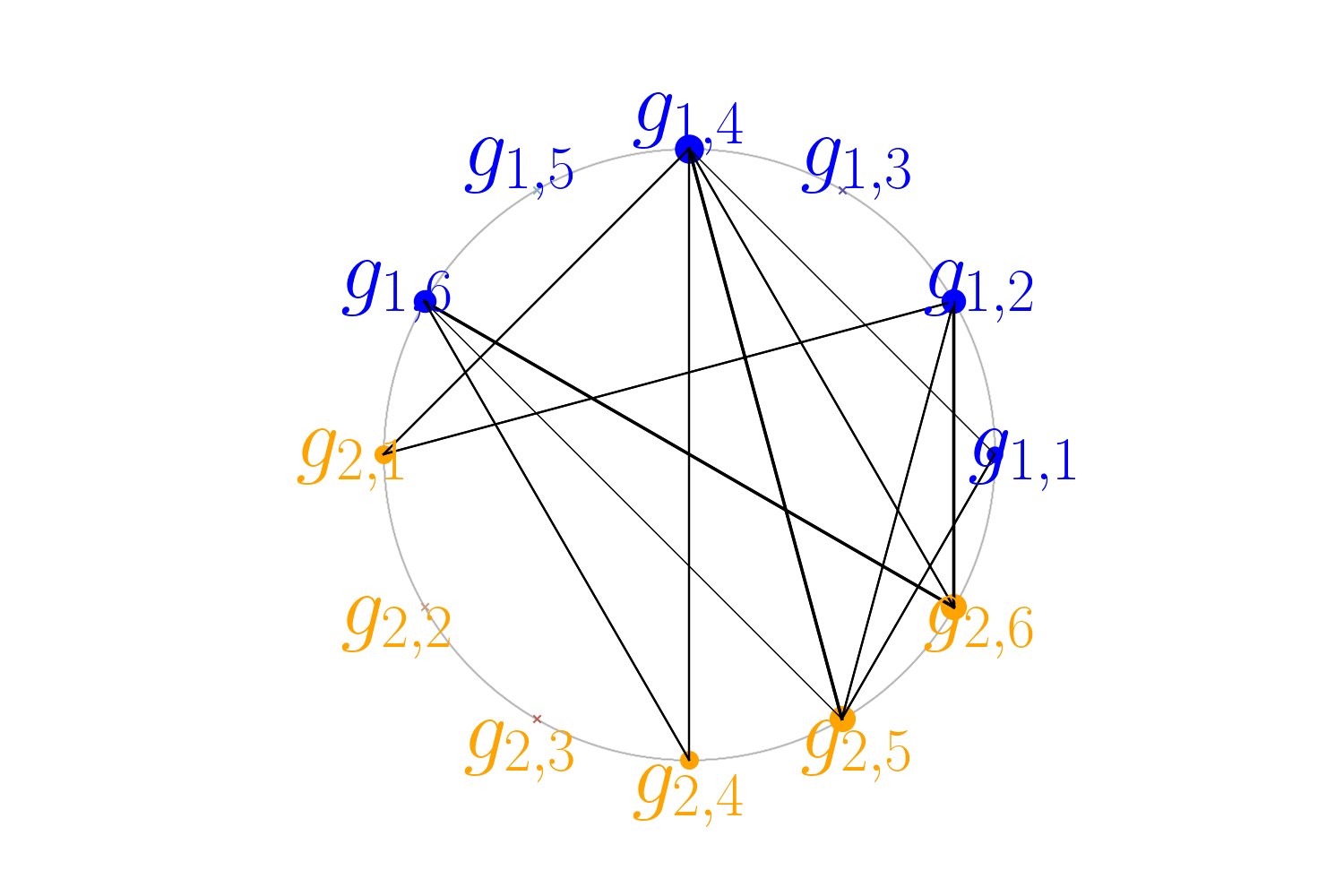}\label{fig:malonaldehyde-watch}}\hfill
\caption{Results from molecular dynamics data. \ref{fig:ethanol-diagram}, \ref{fig:malonaldehyde-diagram} show bond diagrams for ethanol and malonaldehyde, respectively. \ref{fig:ethanol-cosine-colored} and \ref{fig:malonaldehyde-cosine-colored} show the heatmap of cosines (incoherences) of dictionary functions. The color is darker when there is more colinearity.
\ref{fig:ethanol-regularizationpath}, \ref{fig:malonaldehyde-regularizationpath} are regularization paths for a single replicate of ethanol and malonaldehyde. Note that in both figures there are a redundant trajectory of two functions that are added together. 
\ref{fig:ethanol-watch}, \ref{fig:malonaldehyde-watch} Selection of pairs of functions for ethanol and malonaldehyde over replicants using \tsalg~. The node point on the circles represents all functions in the dictionary and the number along the lines are frequencies of each pairs selected over 25 repetitions. \ref{fig:ethanol-watch} means in all 25 repetitions, \tsalg~ selects $g_{1,1}$ and $g_{2,1}$, which are the bond torsions around C-C bond and C-O bond respectively. \ref{fig:malonaldehyde-watch} show that in 24 out of 25 replicates, \tsalg~ is able to select one function from each highly colinear function group.}
\label{fig:results}
\end{figure*}

\section{Discussion and Related Work}
\label{sec:related}

Our method has several good properties.
As long as the dictionary is constructed from some functions that have meaning in the domain of the problem, then our learned embedding is {\em interpretable} by definition.
Furthermore, as discussed in Section \ref{sec:bg}, the mapping $f_S$ is smooth, (implicitly) invertible, and can be naturally extended to values $\xi\in\M$ not in the data.
Finally, our method is flexible with respect to a range of non-linearities.

These features contrast with standard approaches in non-linear dimension reduction.
Parametrizing high-dimensional data by a
small subset of smooth functions has been studied outside the context
of manifold learning as {\em
  autoencoders} \citep{Goodfellow-et-al-2016}. Early work on parametric manifold learning
includes \citet{saulRoweis:03jmlr} and \citet{tehRoweis:nips}, who
proposed a mixture of local linear models whose coordinates are
aligned. In a {\em non-parametric} setting, LTSA \citep{ZhangZ:04} also
gives a global parametrization by aligning locally estimated tangent
spaces. When principal eigenvectors of the Laplace-Beltrami operator
on the manifold are used for embedding, like in Diffusion Maps
\citep{CoiLafLeeMag05}, it can be shown
\citep{Portegies:16} that in the limit of large $n$, with properly
selected eigenfunctions and geometric conditions on the manifold, the
eigenfunctions provide a smooth embedding of the manifold to Euclidean
space. 
However, both the parametric and non-parametric methods above produce learned embeddings $f$ that are abstract in the sense that they do not have a concise functional form.
In this sense, we draw a parallel between our approach and {\em factor models} \citep{yalcin2001}.

Group Lasso type regression for gradient-based variable selection was previously explored in \citet{Haufe2009-yt} and \citet{Ye2012}, but both have a simpler group structure, and are not utilized in the setting of dimension reduction.
More recently, so-called {\em symbolic regression} methods such as \citet{Brunton-2016dt}, \citet{Rudy2019-tk}, and \citet{Champion2019-lu} have been used for linear, non-linear, and machine-learned systems, respectively, and these methods may regarded as  univariate relatives of our approach, since they are concerned with dynamics through time, while we consider the data manifold independently of time.

We also draw several distinctions between the \tsalg~method and the \ouralg~method in \citet{arxivVersion}.
First, \ouralg~uses the same essential idea of sparse linear regression in gradient space, but in order to explain individual embedding coordinate functions.
In contrast, we have no consistent matching between unit vectors in $\bm{I_d}$, and so can only provide an overall regularization path, rather than one corresponding to individual tangent basis vectors. The tangent bases are not themselves gradients of a known function, and, indeed it may not be the case that such a function even exists. Second, \tsalg~method dispenses {\em with the entire Embedding algorithm}, Riemannian metric estimation, and pulling back the embedding gradients steps in \ouralg~, while providing almost everything a user can get from \ouralg~. Apart from simplification, \tsalg~ can be run on $n'\ll n$ data points, about $1/500$ of the data in our experiments (Table 1), while the algorithm in \ouralg~ computes an embedding from all data points. Hence, all operations before the actual GroupLasso are hundreds of times faster than in \ouralg~.
Theoretically, \citet{arxivVersion} only provides (i)analysis in function spaces, and (ii) recovery guarantees for the final step, GroupLasso, based on generic assumptions about the noise. Our paper has {end-to-end guarantees of recovery guarantees from a sample} in Section \ref{sec:theory}.

The reliance on domain prior knowledge in the form of the dictionary $\F$ is essential for \tsalg, and can be a restriction to its usability in practice, especially given the restrictions on gradient field colinearity. However, as the experiments have illustrated, there are domains where construction of a dictionary is reasonable, and explaining the behavior of organic molecules in terms of torsions and planar angles is common in chemistry and drug design \citep{addicoatcollins:2010, Huang2016-nc}. More generally, it would be desirable to utilize a completely agnostic dictionary that also contained the features themselves, and so development of an optimization strategy capable of handling the large amount of colinearity intrinsic to such a set-up is an active area of research.

\section*{Acknowledgement}
The authors acknowledge the support from NSF DMS award 1810975 and DMS award 2015272. The authors also thank the Tkatchenko lab,  and especially Stefan Chmiela for providing both data and expertise. This work is completed at the Pure and Applied Mathematics (IPAM). Marina also gratefully acknowledges a Simons Fellowship from IPAM to her, which made her stay during Fall 2019 possible.

\bibliographystyle{plainnat}
\bibliography{gradients,manifolds,materials,mmp,randos}


\appendix

\onecolumn
\section*{Supplementary Materials}
\newcommand{\bb}{\mathbf{B}}
\newcommand{\eb}{\mathbf{E}}
\newcommand{\degree}{^\circ}

\section{Proofs}
In this section we will provide proofs to the theoretical results in the main text. 

\subsection{Independence of Tangent Basis Selection}
\begin{proposition}
    Consider alternative bases $\mathbf{T}_i' = \mathbf{T}_i\mathbf{\Gamma}_i$ where $\mathbf{\Gamma}_i$ are $d\times d$ orthonormal matrices. If $\{\mathbf{B}_i\}_{i=1}^n$ minimizes \eqref{eq:obj}, then in the new tangent bases, $\{\mathbf{B}_i\mathbf{\Gamma}_i\}_{i=1}^n$ minimizes the corresponding loss function, which is constructed through replacing $\mathbf{X}_i$ by $\mathbf{\Gamma}_i\mathbf{X}_i$ in \eqref{eq:obj}. Furthermore, the selected support $S$ is independent of the basis chosen for each tangent space.
\end{proposition}
\begin{proof}[Proof of Proposition 2]
It suffices to show that the loss in \eqref{eq:obj} does not change
under orthogonal transformation of individual tangent bases. As long as this holds,
$\mathbf{B}_i\mathbf{\Gamma}_i$ must minimize the loss since otherwise
one could argue that $J_{\lambda_n} (\mathbf{B})$ is not a minimum value for the
original tangent space bases. Note that the norm
$\lvert\lvert{\mathbf{B}_{.j}}\lvert\lvert_2$ is unitary invariant. This is because $\mathbf{B}_{.j}=(j-\text{th row of \ }\mathbf{B}_i)_{i=1}^n$ is constructed by stacking the $j-$th row of each $\mathbf{B}_i$. Hence the new norm is  given by the norm of $(j-\text{th row of \ }\mathbf{B}_i\mathbf{\Gamma}_i)_{i=1}^n$;
therefore the Group Lasso penalty doesn't change after changing
$\mathbf{B}_i$ to $\mathbf{B}_i\mathbf{\Gamma}_i$ for each
$i$. Finally, it holds that $\lvert\lvert \mathbf{I}_d -
\mathbf{\Gamma}_i^\top\mathbf{X}_i\mathbf{B}_i\mathbf{\Gamma}_i\lvert\lvert_F^2
= \lvert\lvert
\mathbf{\Gamma}_i^\top\left(\mathbf{I}_d-\xb_i\mathbf{B}_i\right)\mathbf{\Gamma}_i
\lvert\lvert_F^2 = \lvert\lvert \mathbf{I}_d-\xb_i\mathbf{B}_i
\lvert\lvert_F^2$,
so the $\ell_2$-loss is not changed under orthonormal transformation of the tangent bases. These rotation invariances guarantee the same support $S$. 
\end{proof}

\subsection{Proof of Proposition 3}

We start by stating the following lemma, which gives the sufficient and necessary condition of certain matrices $\bb_i$ to be the solution to problem \eqref{eq:obj}.  It also provides conditions on unique support recovery and unique solutions. The proof is standard in convex analysis literature; we follow a procedure as in \citep{Wainwright:2009sharp}.

\begin{lemma}
\begin{enumerate}
    \item Matrix $\bb$ is the optimal solution to problem \eqref{eq:obj} if and only if there exists an matrix $Z=(z_1^\top,z_2^\top,\cdots,z_p^\top)^\top\in\rrr^{p\times nd}$ such that
\begin{equation}
    z_{j}=\begin{cases}
    \frac{\beta_i}{\norm{\beta_i}}\quad &\beta_i\neq 0 \\
    \in\rrr^{nd} \text{\ with\ } \norm{z_j}_2\leq 1,&\text{otherwise}
    \end{cases}
    \label{eq:zcondition}
\end{equation}
and 
\begin{equation}
    \begin{pmatrix}\xb_1^\top(\mathbf{I}_d-\xb_1\bb_1),\xb_2^\top(\mathbf{I}_d-\xb_2\bb_2),\cdots,\xb_n^\top(\mathbf{I}_d-\xb_n\bb_n)\end{pmatrix} = \frac{{\lambda_n}}{\sqrt{nd}} \mathbf{Z}
    \label{eq:kktcondition}
\end{equation}.
\item If under the setting of (a), further in \eqref{eq:zcondition}, we have $\norm{z_i}<1$ whenever $\beta_i = 0$, then all optimal solutions $\widetilde{\bb}$ of problem \eqref{eq:obj} will have support $S(\widetilde{\bb})\subset S(\bb)$.
\item Under setting of (a) and (b). Let $\mathbf{X}_{iS(\bb)}$ be the submatrix constructed by the $S(\bb)$ columns of of $\xb_i$. If all $\mathbf{X}_{iS(\bb)}^\top\mathbf{X}_{iS(\bb)}$ are invertible, then the solution of problem problem \eqref{eq:obj} is unique.
\end{enumerate}
\label{lem:optimal}
\end{lemma}

\begin{proof}
Before we further explore the result, we transform the problem \eqref{eq:obj}. We stack the matrices at each point together. We will now write 
\begin{equation}
    \xb = \begin{pmatrix}\xb_1 \\ \xb_2 \\ \cdots \\ \xb_n \end{pmatrix}\in \rrr^{nd \times p},\quad \bb = \begin{pmatrix} \bb_1, \bb_2, \cdots, \bb_n \end{pmatrix} \in \rrr^{p\times nd}
\end{equation}
Then $\beta_j$ is the $j-th$ row for $\bb$.
Further let matrix
\begin{equation}
    \eb_i = \begin{pmatrix}\mathbf{0}, \cdots, \mathbf{0}, \mathbf{I}_d, \mathbf{0},\cdots,\mathbf{0} \end{pmatrix}^\top \in \rrr^{nd\times d}
\end{equation}
be the block matrix with the $i-th$ block being identity matrix and the other blocks are all zeros. Then the loss function of TSLasso can be rewritten as
\begin{equation}
    J_{{\lambda_n}}(\bb) = \frac{1}{2}\sum_{i=1}^n \norm{\eb_i^\top(\mathbf{I}_{nd}-\xb\bb)\eb_i}_F^2+\frac{{\lambda_n}}{\sqrt{nd}}\norm{\bb}_{1,2}
    \label{eq: rewrite}
\end{equation}
where $\norm{\bb}_{1,2}$ is the norm defined by $\sum_{j=1}^p\norm{\beta_j}_2$.

The proof of this lemma is standard technique in convex analysis. Define $h_i(\bb) = \norm{\eb_i^\top(\mathbf{I}_{nd}-\xb\bb)\eb_i}_F^2$ penalty part and $g$ is the group lasso penalty.

The first step is to compute the gradient of $h_i(\bb)$ with respect to $\bb$. For any $\mathbf{H}\in \rrr^{p\times nd}$, compute
\begin{align}
    &\quad h_i(\bb+\mathbf{H})-h_i(\bb) \\
    &= \trace(\eb_i^\top(\mathbf{I}_{nd}-\xb(\bb+\mathbf{H}))\eb_i)^\top(\eb_i^\top(\mathbf{I}_{nd}-\xb(\bb+\mathbf{H}))\eb_i) - \trace(\eb_i^\top(\mathbf{I}_{nd}-\xb\bb)\eb_i)^\top(\eb_i^\top(\mathbf{I}_{nd}-\xb\bb)\eb_i)\\ 
    &=-2\trace(\mathbf{H}^\top\xb^\top\eb_i\eb_i^\top(\mathbf{I}_{nd}-\xb\bb)\eb_i\eb_i^\top) + O(\norm{\mathbf{H}}_F^2) \\
    &=-2 \left\langle \mathbf{H}, \xb^\top\eb_i\eb_i^\top(\mathbf{I}_{nd}-\xb\bb)\eb_i\eb_i^\top \right \rangle_F + O(\norm{\mathbf{H}}_F^2)
\end{align}
Hence we can conclude that $\nabla_\bb h_i(\bb) = -2\xb^\top\eb_i\eb_i^\top(\mathbf{I}_{nd}-\mathbf{X}\mathbf{B})\mathbf{E}_i\eb_i^\top=-2\mathbf{X}^\top\eb_i(\mathbf{I}_d-\xb_i\bb_i)\eb_i^\top$, and therefore
\begin{align}
\notag
    \nabla_\bb \frac{1}{2}\sum_{i=1}^n\norm{\mathbf{I}_d-\mathbf{X}_i\bb_i}_F^2 &= \sum_{i=1}^n-\xb^\top\eb_i(\mathbf{I}_d-\xb_i\bb_i)\eb_i^\top \\
    &=-\begin{pmatrix}\xb_1^\top(\mathbf{I}_d-\xb_1\bb_1),\xb_2^\top(\mathbf{I}_d-\xb_2\bb_2),\cdots,\xb_n^\top(\mathbf{I}_d-\xb_n\bb_n)\end{pmatrix}.
\end{align}

Recall that we use $\beta_i$ to denote the $i-th$ row of $\bb$. We use a similar argument in proof of lemma 2 of \citep{GO2011} and notice that the original optimization problem is convex and strictly feasible (hence strong duality holds). The primal problem is 
\begin{align}
    \min_{\substack{\bb\in\rrr^{p\times nd}\\ b\in \rrr^p}} &\frac{1}{2}\sum_{i=1}^n\norm{\eb_i^\top(\mathbf{I}_{nd}-\xb\bb)\eb_i}_F^2+\frac{{\lambda_n}}{\sqrt{nd}}\sum_{j=1}^pb_j\\
    s.t. \ &(\beta_j,b_j)\in \mathcal{K},1\leq j\leq p
\end{align}
where $\mathcal{K}$ is the second-order cone as usually defined. The dual problem is given by
\begin{align}
    \max_{\substack{\mathbf{Z}\in\rrr^{p\times nd} \\ t\in\rrr^p}}\min_{\substack{\bb\in\rrr^{p\times nd}\\b\in \rrr^p}} L(\bb,b,\mathbf{Z},t)&=\frac{1}{2}\sum_{i=1}^n\norm{\eb_i^\top(\mathbf{I}_{nd}-\xb\bb)\eb_i}_F^2+\frac{{\lambda_n}}{\sqrt{nd}}\sum_{j=1}^pb_j+\sum_{j=1}^p\langle (z_j,t_j),(\beta_j,b_j) \rangle\\ s.t.\ &(z_j,t_j)\in \mathcal{K}\degree
\end{align}
where $z_j\in\rrr^{nd}$ is the $j-$th row of $\mathbf{Z}$. Note that $\mathcal{K}\degree$ is the polar cone of $\mathcal{K}$ and second order cone is self-dual. Hence we have $(z_i,-\mathbf{T}_i)\in \mathcal{K}$. 

Since the primal problem is strictly feasible, strong duality holds. For any pair of ($\bb^*,b^*$) and ($\mathbf{Z}^*,t^*$) primal and dual solutions, they have to satisfy the KKT condtion that
\begin{subequations}\label{eq:kkt}
\begin{align}
    \norm{\beta_j^*}_2&\leq b_j^*, & 1\leq j \leq p\; ,\label{eq:primalfeasibility} \\
    \norm{z_j^*}_2&\leq -t_j^*,& 1\leq j \leq p\; , \label{eq:dualfeasibility}\\
    z_j^{*T}\beta_j^*+t_j^*b_j^* &= 0, & 1\leq j \leq p \; , \label{eq:complementaryslackness} \\
    \nabla_B\left[\frac{1}{2}\sum_{i=1}^n\norm{\mathbf{I}_d-\xb_i\bb_i}_F^2\right]+\mathbf{Z}^* &= 0 \; , \label{eq:gradientB} \\
    \frac{{\lambda_n}}{\sqrt{nd}}+t_j^* &= 0\; . \label{eq:gradientb}
\end{align}
\end{subequations}
Note that \eqref{eq:complementaryslackness}implies that $t_j^* = -\frac{{\lambda_n}}{\sqrt{nd}}$. Then by \eqref{eq:primalfeasibility} and \eqref{eq:dualfeasibility} we have $\norm{z_j^{*T}\beta_j^*}\leq\frac{{\lambda_n}}{\sqrt{nd}}\norm{\beta_j}_2$. Notice that the equality holds in \eqref{eq:complementaryslackness}, there fore $\norm{z_j^*}=\frac{\sqrt{nd}}{{\lambda_n}}$ and $b_j^* = \norm{\beta_j^*}$. Renormalize $z_j^* = \frac{\sqrt{nd}}{{\lambda_n}}z_j^*$ and part (a) holds. For part b, for any $j, z_j^{*T}\beta_j = \norm{\beta_j}_2$.  Then $\beta_j = 0$ must hold for any $\norm{z_j}<1$. For part (c) note that in this case the loss function is strictly convex when the original problem is restricted to minimizing over $\bb:\beta_i = 0,\quad \forall i\notin S(\bb)$. This strict convexity implies the uniqueness of solution.
\end{proof}

The previous lemma provides a tool for understanding the support recovery consistency of \tsalg. 

For any arbitrary $S\subset[p]$ such that $|S| = d, \rank \xb_{iS} = d$ holds for all $i\in[n]$, we establish a sufficient condition on $\xb_{iS}$ such that they can be discovered by the \tsalg. 
Suppose at each data point $i$, we decompose the matrix $\mathbf{I}_d$ by 
\begin{equation}
    \mathbf{I}_d = \xb_{iS}\bb_{iS}^* + \mathbf{W}_{iS}
    \label{eq:noise}
\end{equation}
where $\bb_{iS}^*$s are $p\times d$ matrices that only has non zero entries in rows in $S$ and minimizes the loss $\norm{\mathbf{I}_d-\xb_i\bb_i}_F^2$. In fact, since $\xb_{iS}$ is full rank, there exists a unique $\bb_{iS}^*$ for each $i$ such that $\mathbf{W}_{iS} = 0$. Denote $\bb_{i,j\cdot}^*$ be the $j-$th row in $\bb_{iS}^*$ and define
\begin{equation}
    \tilde{b}_S=\min_{i\in[n]} \min_{j\in S} \norm{\bb_{i,j\cdot}^*}\;,
    \label{eq:b}
\end{equation}
This is a sample version of $b_S$ defined in \eqref{eq:b-global}.

The following lemma shows a sufficient condition on $\mathbf{X}_i$ so that the true support can be found. We first define several derived quantities of $\xb_i$. Denoting the $j-$th column of matrix $\xb_i$ by $x_{ij}$, we define
\begin{subequations}
\begin{align}
    \text{S-incoherence} \quad & \tilde{\mu}_S = \max_{i=1:n,j\in S,j'\notin S} \frac{|x_{ij}^\top x_{ij'}|}{\norm{\nabla f_j(\xi_i)}\norm{\nabla f_{j'}(\xi_i)}} \label{eq:mu} \\
    \text{internal-colinearity} \quad & \tilde{\nu}_S =\max_{\i=1:n} \norm{(\tilde{\xb}_{iS}^\top\tilde{\xb}_{iS})^{-1}-\mathbf{G}_S(\xi_i)^{2}} \label{eq:nu}. \\
    \text{maximal gradient norm}\quad &\tilde{\phi}_S = \max_{i=1:n} \max_{j\in S} \norm{\nabla f_j(\xi_i)}\label{eq:phi}
\end{align}
\end{subequations}

These are sampled version of $\mu_S$,$\nu_S$ and $\phi_S$ defined on the whole manifold from \eqref{eq:mu-global},\eqref{eq:nu-global} and \eqref{eq:phi-global}. 

Now we are ready to prove proposition \ref{prop:recovery}. We start with some lemmas in linear algebra.

\begin{lemma}
    Let $\mathbf{A},\mathbf{B}$ be $d\times d$ positive definite matrices. Then $\norm{\mathbf{A}^{-1}-\mathbf{B}^{-1}}\leq \norm{\mathbf{B}^{-1}}^2\norm{\mathbf{A}-\mathbf{B}}$
    \label{lem:psdinvperturbation}
\end{lemma}
\begin{lemma}
    Let $\mathbf{A},\mathbf{B}$ be two $d\times d$ matrices. $\mathbf{A}$ is positive semidefinite. Denote $\norm{\mathbf{A}}_{\infty,2}$ be the maximum $\ell_2$ norm of the rows of $\mathbf{A}$. Then $\norm{\mathbf{A}\mathbf{B}}_{\infty,2}\leq \norm{\mathbf{A}} \norm{\mathbf{B}}_F$
    \label{lem:infinityl2norm}
\end{lemma}
\begin{proof}
    Write $\mathbf{A}=(a_{ij})_{d\times d},\mathbf{B}=(b_{ij})_{d\times d}$, then by definition
    \begin{align*}
        \norm{\mathbf{A}\mathbf{B}}_{\infty,2}^2 &= \max_{i=1:d} \sum_{j=1}^d \left(\sum_{k=1}^d a_{ik}b_{kj}\right)^2 \\
        &\leq \max_{i=1:d} \sum_{j=1}^d \left(\sum_{k=1}^d a_{ik}^2\right)\left(\sum_{k=1}^d b_{kj}^2\right) \\
        &\leq \left(\max_{i=1:d}\sum_{k=1}^d a_{ik}^2\right) \left(\sum_{j=1}^d\sum_{k=1}^d b_{kj}^2\right) \\
        &= \norm{\mathbf{A}}_{\infty,2}^2\norm{\mathbf{B}}_F^2\;.
    \end{align*}
    Since $\mathbf{A}$ is positive semidefinite, we have
    \begin{align*}
        \norm{\mathbf{A}}_{\infty,2}^2&=\max_{i=1:d} \left(\mathbf{A}\mathbf{A}\right)_{ii} \leq \norm{\mathbf{A}^2}=\norm{\mathbf{A}}^2\;.
    \end{align*}
    Hence we conclude the desired result.
\end{proof}

\begin{lemma}
    Let $\delta = \min_{\xi \in \M}\min_{j=1:p}\norm{\nabla f_j(\xi)}$, then
    $\norm{(\xb_{iS}^\top\xb_{iS})^{-1}}_2 \leq 1+\frac{\tilde{\nu}_S}{\delta^2}$
    \label{lem:xsinveigen}
\end{lemma}
\begin{proof}
    Recall that $\mathbf{G}_S(\xi_i)=\diag\{ \norm{\nabla f_j(\xi_i)}\}_{j,j'\in S}$
    We first consider that 
    \begin{align}
    &\norm{(\xb_{iS}^\top\xb_{iS})^{-1}-\mathbf{I}_d}_2\\
     =&\norm{\mathbf{G}_S^{-1}(\xi_i)(\tilde{\xb}_{iS}^\top\tilde{\xb}_{iS})^{-1}\mathbf{G}_S(\xi_i)^{-1}-\mathbf{G}_S^{-1}(\xi_i)\mathbf{G}_S(\xi_i)^{2}\mathbf{G}_S^{-1}(\xi_i)} \\
     \leq & \norm{(\tilde{\xb}_{iS}^\top\tilde{\xb}_{iS})^{-1}-\mathbf{G}_S(\xi_i)^{2}}\norm{\mathbf{G}_S^{-1}(\xi_i)}^2 \\
     \leq & \frac{\tilde{\nu}_S}{\delta^2}
    \end{align}
    And the desired results come from triangular inequality.
\end{proof}

\begin{lemma}
Let $\{\xi_i\}_{i=1}^n$ be fixed data points on $\M$. Let $\tilde{\delta} = \min_{\xi\in\M}\min_{j=1:p}\norm{\nabla f_j(\xi)}$ and $\Gamma = \max_{\xi\in\M}\max_{j=1:p}\norm{\nabla f_j(\xi)}$
Let $\tilde{\mu}_S,\tilde{\nu}_S,\tilde{\phi}_S$ defined from $\xb_{iS}$ according to \eqref{eq:mu},\eqref{eq:nu} and \eqref{eq:phi} respectively. Then Tangent Lasso problem \eqref{eq:obj} has a unique solution $\widehat{\bb}=[\widehat{\bb}_1,\widehat{\bb}_2,\cdots,\widehat{\bb}_n]\in \rrr^{p\times nd }$ with support $S(\widehat{\bb})$ included in the true support $S$ if $(1+\frac{\tilde{\nu}_S}{\delta^2})^2\tilde{\mu}_S\tilde{\phi}_S\Gamma d<1 $. Furthermore, if $\lambda_n(1+{\tilde{\nu}_S}/{\delta^2})^2< \tilde{b}_S\sqrt{n}/2$, then $S(\widehat{\bb})=S$.
\label{lem:recover-nonoise}
\end{lemma}

\begin{proof}
 We follow the procedure of Primal-Dual witness method (see e.g.\citet{Wainwright:2009sharp}, \citet{GO2011}, \citet{Elyaderani2017-ce}). 
 
Still considering the reformulated optimization problem \eqref{eq: rewrite}, we first find $\widehat{\bb}$ from minimizing a restricted optimization problem
\begin{equation}
    \min_{S(\bb)\subset S} J_{\lambda_n}(\bb) = \frac{1}{2}\sum_{i=1}^n \norm{\eb_i^\top(\mathbf{I}_{nd}-\xb\bb)\eb_i}_F^2+\frac{\lambda_n}{\sqrt{nd}}\norm{\bb}_{1,2}.
\end{equation}

We then construct a dual solution $\widehat{\mathbf{Z}}$ and show that $\widehat{\bb}$ is the solution to the original optimization problem.
We write $z_j$ as the $j-$th row of $\widehat{\mathbf{Z}}$ and decompose each $\widehat{z}_j = [\widehat{z}_{j,1},\widehat{z}_{j,2},\cdots,\widehat{z}_{j,n}]$. According to lemma \ref{lem:optimal}, we can solve for $\widehat{\mathbf{Z}}$ from those optimality conditions.

First, notice that
\begin{equation}
    \widehat{\bb}_{iS}-\bb_{iS}^*=-\frac{\lambda_n}{\sqrt{nd}}(\xb_{iS}^\top\xb_{iS})^{-1}\widehat{\mathbf{Z}}_{S,i}\;.
    \label{eq:decompositionS}
\end{equation}
where $\widehat{\mathbf{Z}}_S$ is constructed by concatenating the $j\in S$ row of $\widehat{\mathbf{Z}}_i$.

 For an $d\times d$ matrix $\mathbf{A}$, we write $\norm{\mathbf{A}}_{\infty,2}= \max_{i=1}^d \norm{a_{i}}_2$, where $a_i$ is the $i-$th row of $A$. Then it holds that from lemma \ref{lem:infinityl2norm} 
 \begin{equation}
     \norm{(\xb_{iS}^\top\xb_{iS})^{-1}\widehat{\mathbf{Z}}_{S,i}}_{\infty,2} \leq \norm{(\xb_{iS}^\top\xb_{iS})^{-1}}\norm{\widehat{\mathbf{Z}}_{S,i}}_{F}.
 \end{equation}
 
Therefore recall that $\norm{\widehat{\mathbf{Z}}_S}_{\infty,2} = 1$ we conclude that $ \norm{\widehat{\mathbf{Z}}_{S,i}}_{F} \leq \sqrt{d}$. And adopting lemma \ref{lem:xsinveigen} we have 
\begin{equation}
    \norm{(\xb_{iS}^\top\xb_{iS})^{-1}\widehat{\mathbf{Z}}_{S,i}}_{\infty,2} \leq \sqrt{d}(1+\frac{\tilde{\nu}_S}{\delta^2})^2
\end{equation}

According to \eqref{eq:decompositionS} and the assumption, ${\lambda_n}\sqrt{d}(1+\frac{\tilde{\nu}_S}{\delta^2})^2/\sqrt{nd}< \frac{1}{2}\tilde{b}_S$, then $\norm{\widehat{\bb}_{iS,j\cdot}}\geq \frac{1}{2}\tilde{b}_S$ for each row $j \in S$.

On the other hand, for any $j'\notin S$,we have 
\begin{equation}
    \widehat{z}_{j',i}=x_{ij'}^\top\xb_{iS}(\xb_{iS}^\top\xb_{iS})^{-1}\widehat{\mathbf{Z}}_{S,i} .
\end{equation}
It suffices to verify that $\norm{\widehat{z}_j}<1$ for all $j' \notin S$. For any $i$, we have
\begin{equation}
    \norm{x_{ij'}^\top\xb_{iS}(\xb_{iS}^\top\xb_{iS})^{-1}}_2\leq (1+\frac{\tilde{\nu}_S}{\delta^2})^2 \norm{x_{ij'}^\top\xb_{iS}}_2 \leq \sqrt{d}(1+\frac{\tilde{\nu}_S}{\delta^2})^2\tilde{\mu}_S \norm{\nabla f_{j'}(\xi_i)}\max_{j\in S}\norm{\nabla f_j(\xi_i)}\leq \sqrt{d}(1+\frac{\tilde{\nu}_S}{\delta^2})^2\tilde{\mu}_S\tilde{\phi}_S\Gamma
\end{equation}

Directly compute that
\begin{align*}
    \norm{\widehat{z}_{j'}}^2 &\leq \sum_{i=1}^n \norm{x_{ij'}^\top\xb_{iS}(\xb_{iS}^\top\xb_{iS})^{-1}\widehat{\mathbf{Z}}_{S,i}}_2^2 \\
    &\leq \sum_{i=1}^n \norm{x_{ij'}^\top\xb_{iS}(\xb_{iS}^\top\xb_{iS})^{-1}}_2^2\norm{\widehat{\mathbf{Z}}_{S,i}}_F^2 \\
    &\leq d(1+\frac{\tilde{\nu}_S}{\delta^2})^4\tilde{\mu}_S^2\tilde{\phi}_S^2\Gamma^2 \sum_{i=1}^n\norm{\widehat{\mathbf{Z}}_{S,i}}_F^2 \\
    &\leq (1+\frac{\tilde{\nu}_S}{\delta^2})^4\tilde{\mu}_S^2\tilde{\phi}_S^2\Gamma^2d^2 < 1
\end{align*}
\end{proof}

This lemma is the recovery result if the tangent space is estimated without any noise. Note that this conditions also implies further results on the 'isometric' property of TSLasso. If there are two different subsets $S, S'$ such that $|S|=|S'|=d$ and both has rank $d$ at each data point. Then for both subsets, $\xb_{iS}^\top\xb_{iS}$ are invertible, and the lemma also implies that $\tilde{\mu}_S\tilde{\nu}_Sd<1$ cannot hold at the same time for both subsets. The one picked by TSLasso (usually) has a lower value of $\tilde{\nu}_S$, and will be closer to isometry to some extent.

This recovery result does not involve the tuning parameter for false inclusion. Therefore, it justifies our selection of tuning parameter that force the support has cardinality less than $d$. If we do observe $d$ functions selected and they have rank $d$ everywhere, then under incoherence condition they must be a right parameterization. To avoid false exclusion, the tuning parameter ${\lambda_n}$ cannot be too large. 

Now we connect these support recovery results inherent to our optimization approach with the tangent space estimation algorithm. Let $\mathbf{T}_i,\widehat{\mathbf{T}}_i$ be the orthogonal basis in $\rrr^{D\times d}$ for true and estimated tangent space respectively, and write
\begin{equation}
    e = \max_{i=1}^n \norm{\mathbf{T}_i\mathbf{T}_i^\top-\widehat{\mathbf{T}}_i\widehat{\mathbf{T}}_i^\top}_{2}.
    \label{eq:errtang}
\end{equation}
We have the following recovery result in the setting that gradient is estimated with some noise. 

\begin{lemma}
Let $\xi_i,i=1:n$ be fixed data points on manifold $\M\subset \rrr^D$.
Given $S$ a subset of functions in dictionary $\mathcal{F}=\{f_j,j\in[p]\}$ with $|S|=d$. Suppose $\rank \grad f_S = d$ at each data point. Fix $\mathbf{T}_i$ as an orthonormal basis of tangent space at $\xi_i$, and $\widehat{\mathbf{T}}_i$ a basis for the estimated tangent space. And further define $\xb_i = \mathbf{T}_i^\top [\nabla f_j],\widehat{\xb}_i = \widehat{\mathbf{T}}_i^\top[\nabla f_j],j\in [p]$ where $\nabla$ is the ambient gradient. Define $\bb_{iS}^*,\tilde{b}_S$ the same as lemma \ref{lem:recover-nonoise}. Assume that $\norm{\nabla f_j}=1$ for all $\xi_i,i\in[n],j\in[p]$. Define $\tilde{\mu}_S,\tilde{\nu}_S$ from  \eqref{eq:mu} and \eqref{eq:nu} and $e$ from \eqref{eq:errtang}.  Then let $\widehat{\bb}$ be the solution of TSLasso problem
\begin{equation}
    J_{\lambda_n}(\bb) = \frac{1}{2}\sum_{i=1}^n\norm{\mathbf{I}_d-\widehat{\xb}_i\bb_i}_F^2+\frac{{\lambda_n}}{\sqrt{nd}}\norm{\bb}_{1,2}\; ,
\end{equation}
If $(1+{\tilde{\nu}_S}/{\delta^2})^2\tilde{\mu}_S\tilde{\phi}_S\Gamma d<1 $ and $\lambda_n(1+{\tilde{\nu}_S}/{\delta^2})^2< \tilde{b}_S\sqrt{n}/2$, there exists a positive constant $c_0$  such that if $e<c_0$ then $S(\widehat{\bb})=S$.
\label{thm: fixed}
\end{lemma}

\begin{proof}
The proof is direct by identifying the new $\tilde{\mu}_S',\tilde{\nu}_S'$ parameters under noisy estimation of tangent space. The other parameters $\tilde{\phi}_S,\Gamma,\delta$ are not related with tangent spaces and thus remains unchanged.

Denote $\widehat{x}_ij$ the $j-$th column of $\widehat{\xb}_i$. Similarly, to \eqref{eq:mu}, we first bound
\begin{align*}
    \widehat{x}_{ij}^\top\widehat{x}_{ij'} &= \nabla f_j(\xi_i)^\top [\widehat{\mathbf{T}}_i\widehat{\mathbf{T}}_i^\top-\mathbf{T}_i\mathbf{T}_i^\top] \nabla f_j(\xi_i)+ \nabla f_j(\xi_i)^\top \mathbf{T}_i\mathbf{T}_i^\top \nabla f_j(\xi_i) \\
    &\leq \norm{\widehat{\mathbf{T}}_i\widehat{\mathbf{T}}_i^\top-\mathbf{T}_i\mathbf{T}_i^\top}_2 \norm{\nabla f_j(\xi_i)}\norm{\nabla f_{j'}(\xi_i)} + \tilde{\mu}_S  \norm{\nabla f_j(\xi_i)}\norm{\nabla f_{j'}(\xi_i)},\quad \text{for all \ } j\in S, j'\notin S, i\in[n]
\end{align*}
So $\tilde{\mu}_S' \leq  \tilde{\mu}_S + e$.

By definition, let $$\tilde{\widehat{\xb}}_{iS} = \left[ \frac{\widehat{\mathbf{T}}_i^\top\nabla f_j(\xi_i)}{\norm{\nabla f_j(\xi_i)}}\right]_{j\in S} = \widehat{\xb}_{iS}\mathbf{G}(\xi_i)^{-1}\, $$ where $\mathbf{G}(\xi_i) = \diag\{\norm{\nabla f_j(\xi_i)}\}_{j\in S}$  and then we have
\begin{align*}
    \tilde{\nu}_S'&=\norm{(\tilde{\widehat{\xb}}_{iS}^\top\tilde{\widehat{\xb}}_{iS})^{-1} - \mathbf{G}(\xi_i)^{-2}} \leq \tilde{\nu}_S + \norm{(\tilde{\widehat{\xb}}_{iS}^\top\tilde{\widehat{\xb}}_{iS})^{-1} - (\tilde{{\xb}}_{iS}^\top\tilde{{\xb}}_{iS})^{-1}}
\end{align*}
It suffices to upper bound the second term. We can apply lemma \ref{lem:psdinvperturbation}, the perturbation bound of inverse of positive definite matrices. It suffice to compute
\begin{align*}
    \norm{(\tilde{\xb}_{iS}^\top\tilde{\xb}_{iS})^{-1}} &\leq \norm{(\tilde{\xb}_{iS}^\top\tilde{\xb}_{iS})^{-1} - \mathbf{G}_S(\xi_i)^{2} + \mathbf{G}_S(\xi_i)^{2}} \leq \tilde{\phi}_S^2+\tilde{\nu}_S
\end{align*}
And since for any $j,j'\in S$, it holds that 
\begin{align*}
    |(\tilde{\widehat{\xb}}_{iS}^\top\tilde{\widehat{\xb}}_{iS})_{jj'} - (\tilde{{\xb}}_{iS}^\top\tilde{{\xb}}_{iS})_{jj'}| \leq \norm{\mathbf{T}_i\mathbf{T}_i^\top - \widehat{\mathbf{T}}_i\widehat{\mathbf{T}}_i^\top} \leq e
\end{align*}
\begin{align*}
    \norm{\tilde{\widehat{\xb}}_{iS}^\top\tilde{\widehat{\xb}}_{iS} - \tilde{{\xb}}_{iS}^\top\tilde{{\xb}}_{iS}} & \leq \norm{\tilde{\widehat{\xb}}_{iS}^\top\tilde{\widehat{\xb}}_{iS} - \tilde{{\xb}}_{iS}^\top\tilde{{\xb}}_{iS}}_F\leq de
\end{align*}

And thus we have
\begin{align*}
    \norm{(\tilde{\widehat{\xb}}_{iS}^\top\tilde{\widehat{\xb}}_{iS})^{-1} - (\tilde{{\xb}}_{iS}^\top\tilde{{\xb}}_{iS})^{-1}} \leq (\tilde{\phi}_S^2+\tilde{\nu}_S)^2de
\end{align*}
Hence $\tilde{\nu}_S' \leq \tilde{\nu}_S + (\tilde{\phi}_S^2+\tilde{\nu}_S)^2de$

For sufficiently small $e$, we will have $(1+\frac{\tilde{\nu}_S'}{\delta^2})^2\tilde{\mu}_S'\tilde{\phi}_S\Gamma d<1 $ and ${\lambda_n}(1+\frac{\tilde{\nu}_S'}{\delta^2})^2/\sqrt{n}< \frac{1}{2}\tilde{b}_S$ as these two inequality holds when $e=0$. Then lemma \ref{lem:recover-nonoise} guarantees exact recovery.
\end{proof}

\begin{proof}[Proof of Proposition \ref{prop:recovery}]
With probability one, the following comparisons between sample based quantities and whole manifold versions holds:
\begin{equation}
    \tilde{\mu}_S \leq \mu_S,\quad
    \tilde{\nu}_S \leq \nu_S,\quad
    \tilde{\phi}_S \leq \phi_S,\quad 
    \tilde{b}_S \geq b_S
\end{equation}
Let $c_1,c_2$ be the same as $\tilde{c}_1,\tilde{c}_2$ defined in the proof of theorem \ref{thm: fixed}, replacing all sample version quantities  $\mu_S,\nu_S,\phi_S,b_S$ with their  global manifold counterparts $\mu_S,\nu_S,\phi_S,b_S$. 

Then the assumptions of the proposition guarantees that there exists a $c_0$ such that whenever $e<c_0$, exact recovery holds. It suffices to notice that 
\begin{equation}
    P(S(\hat{\bb})=S) \leq P(e < c_0) \geq 1-4\left(\frac{1}{n}\right)^{\frac{2}{d}}
\end{equation}
given by lemma \ref{lem:minimax}.
\end{proof}

\begin{lemma}[Proposition in \citet{aamari2018}]
   For sufficiently large $C$, let $r_N = C(\log n/(n-1))^{{1}/{d}}$, tangent spaces $\hat{\mathbf{T}}_i$ estimated by WL-PCA in section \ref{sec:tangent} with linear kernel satisfy that with probability at least $1-4(1/n)^{2/d}$
   \begin{equation}
       \max_{i=1:n} \norm{\mathbf{T}_i\mathbf{T}_i^\top-\hat{\mathbf{T}}_i\hat{\mathbf{T}}_i^\top}= O(r_N)=O((\frac{\log n}{n-1})^{\frac{1}{d}})\;.
   \end{equation}
   \label{lem:minimax}
 \end{lemma}
\begin{remark}
     Note that in this lemma, the hidden constant in big-$O$ notation is determined by the manifold and sampling density.
\end{remark}
 
\section{Addtional Experimental Results}

We include some additional experimental results and information in this section. The settings of two experiments on synthetic data are shown in table \ref{tab:exps2}.

\begin{table*}[!h]
    \centering
    \begin{tabular}{l|c|c|c|c|c|c|c|c}
    \hline
    \hline
    Dataset       & $n$   & $N_a$ & $D$ & $d$ & $\epsilon_N$ & $n'$ & $p$ & $\omega$ \\
    \hline
    Swiss Roll & 10000 & NA &49 & 2 &.18 & 100 & 51 & 1 \\
    Rigid Ethanol & 10000 & 9 & 50 & 2 & 3.5 & 100 & 12 & 25\\
    \hline
    \hline
    \end{tabular}
    \caption{Parameters in different experiments.}
    \label{tab:exps2}
\end{table*}

\subsection{Results on Swiss Roll Data}
We begin our experimental study by demonstrating that \tsalg~ is invariant to the choice of embedding algorithm on the classic unpunctured SwissRoll dataset. This dataset consists of points sampled from a two dimensional rectangle and rolled up along one of the two axes aFigure \ref{fig:swiss-unrotated} shows the SwissRoll dataset in $\rrr^3$, then randomly rotated in $D = 49$ dimensions. 

The dictionary $\F$ consists of $g_{1,2}$, the two intrinsic coordinates, as well as $g_{j+2} = \xi_j$, for $j = 1, \cdots 49$, the coordinates of the feature space. Applying ManifoldLasso to the embeddings identifies the set $S = \{g_1,g_2\}$ as the manifold parametrization. This successful recovery of parametrizing functions is observed in each replicate. Figure \ref{fig:swiss-reg} shows the regularization path in one replicates.

\begin{figure}
    \centering
    \subfloat[]{\includegraphics[width = 0.45\textwidth]{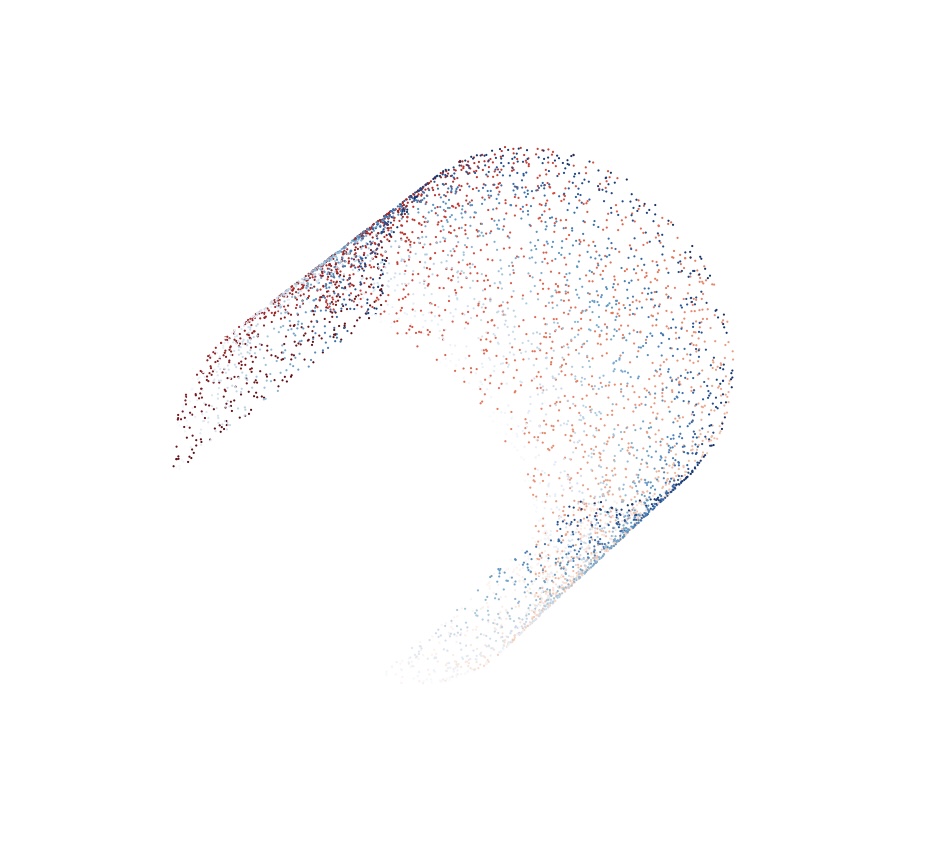}\label{fig:swiss-unrotated}}
    \subfloat[]{\includegraphics[width = 0.45\textwidth]{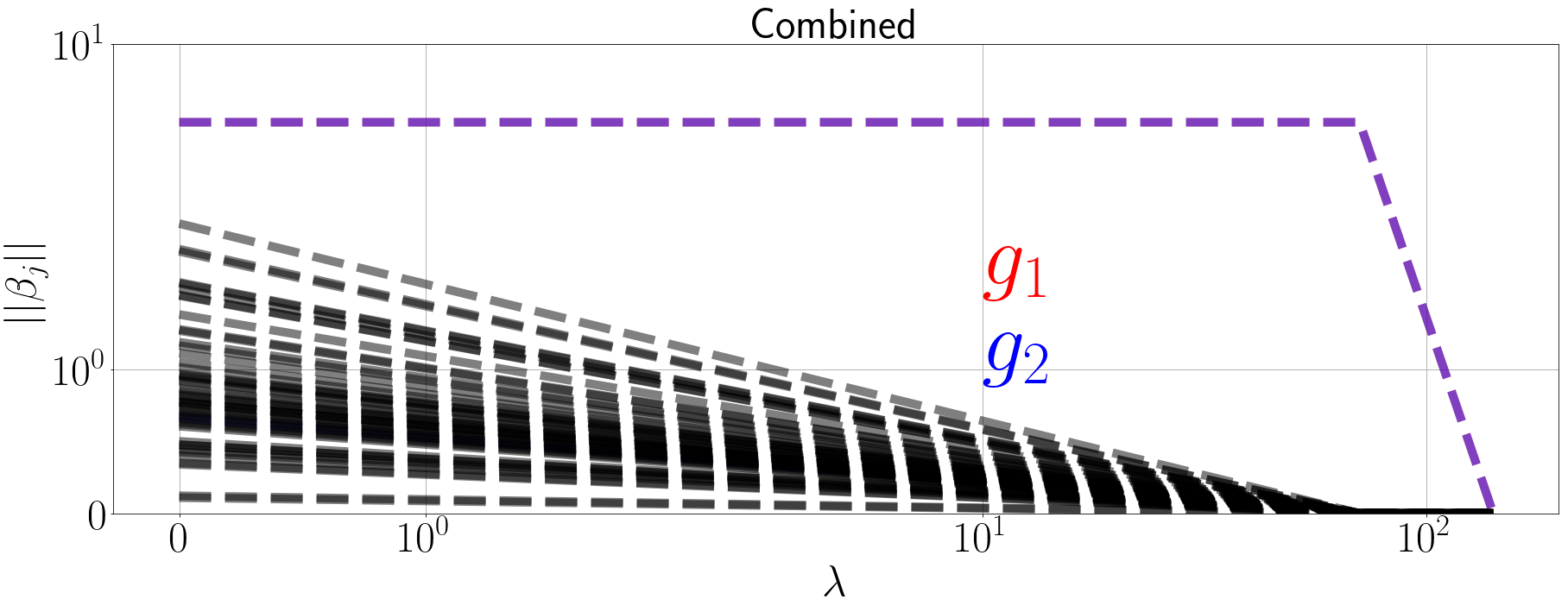}\label{fig:swiss-reg}}
    \caption{Swiss Roll data and result. \textbf{Left:} Unrotated swiss roll dataset in $\rrr^3$. This dataset is then randomly rotated into $\rrr^{49}$. \textbf{Right:} The regularization path of \tsalg~ on SwissRoll datatset in one replicate. Note that in fact there are two functions selected and their regularization path added together.}
\end{figure}

\subsection{Results on Rigid Ethanol Dataset}

We construct an ethanol skeleton composed of the atoms shown in Figure \ref{fig:ethanol-diagram}. We then sample configurations as we rotate the atoms around the C-C and C-O bonds. In contrast with the MD trajectories, which are simulated according to quantum dynamics, these two angles are distributed uniformly over a grid, and Gaussian noise ($N_D(0,\sigma^2I_D)$) is added to the position of each atom. We call the resultant dataset RigidEthanol. As expected given our two a priori known degrees of freedom, Figures \ref{fig:re-nonoise-pca}, \ref{fig:re-nonoise-embedding-g1} and \ref{fig:re-nonoise-embedding-g2} show that the estimated manifold is a two-dimensional surface with a torus topology similar to that observed for the MD Ethanol in Figure \ref{fig:ethanol-diffusion1}. In particular, it is parameterized by bond torsions $g_1$ and $g_2$. The dictionary contains the 12 torsions implicitly defined by the bond diagram, the same as the MDS real data experiment. The function pattern is also the same as the real ethanol dataset. 

\begin{figure}
    \centering
    \subfloat[]{\includegraphics[width = 0.3\textwidth]{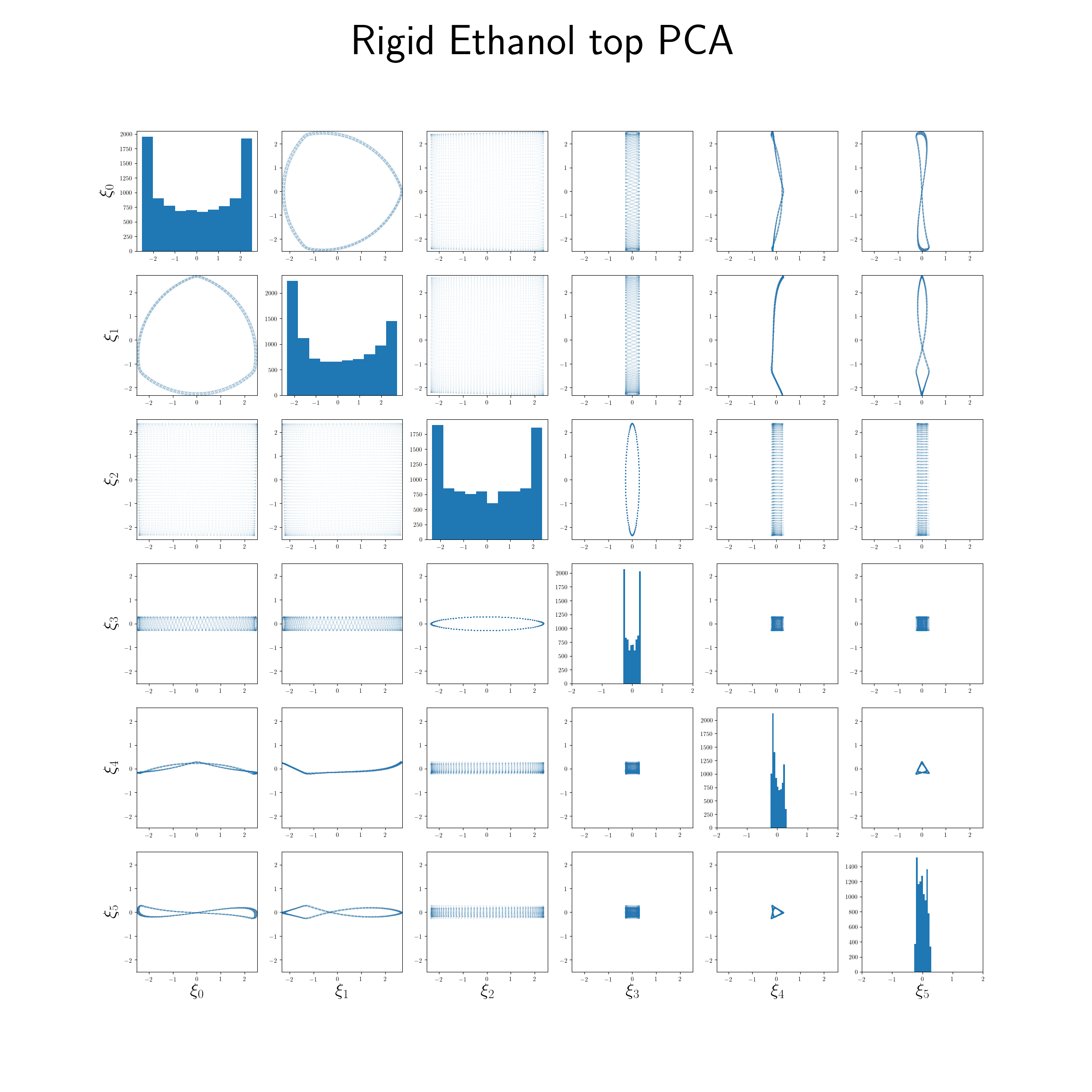}\label{fig:re-nonoise-pca}}
    \subfloat[]{\includegraphics[width = 0.3\textwidth]{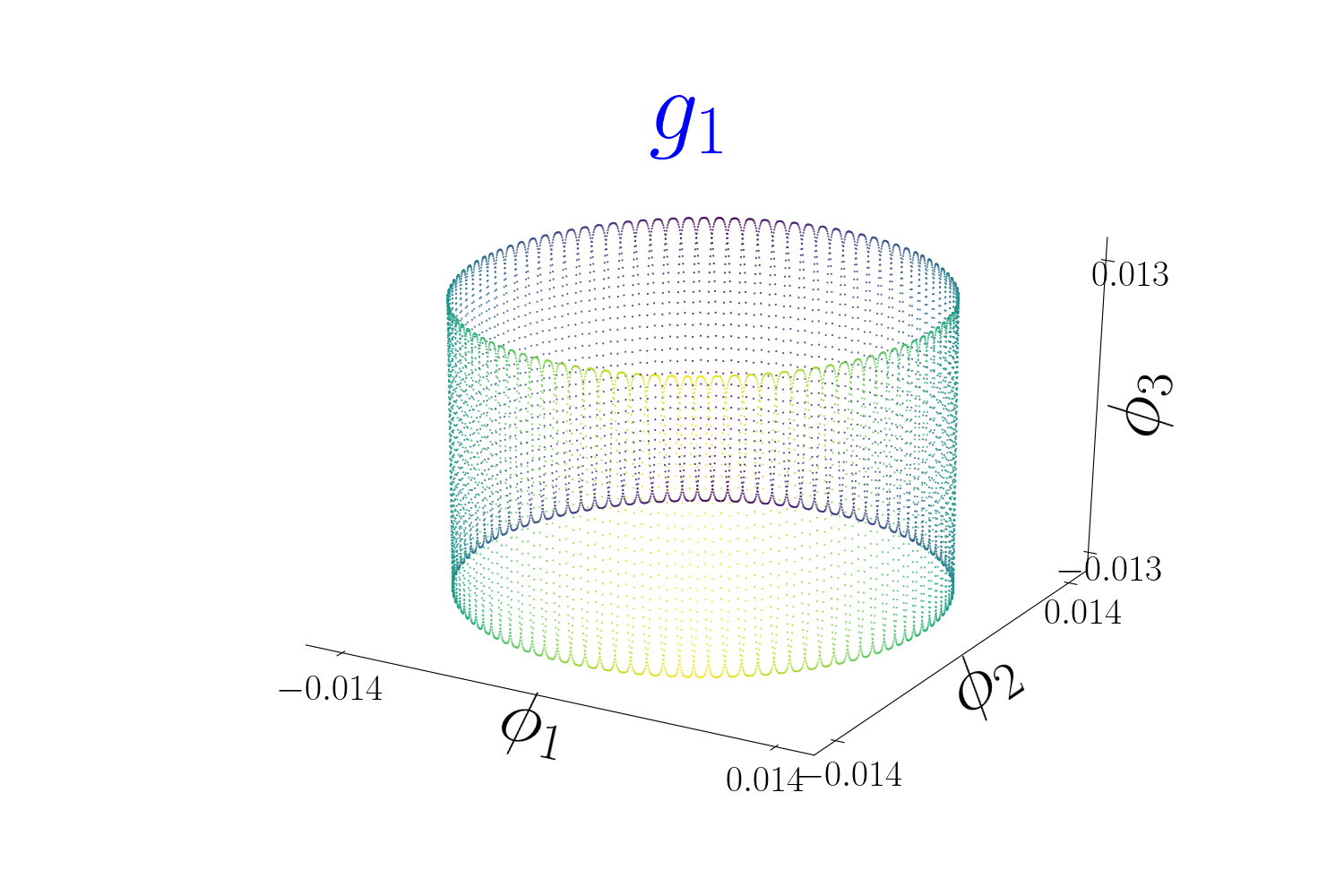}\label{fig:re-nonoise-embedding-g1}}
    \subfloat[]{\includegraphics[width = 0.3\textwidth]{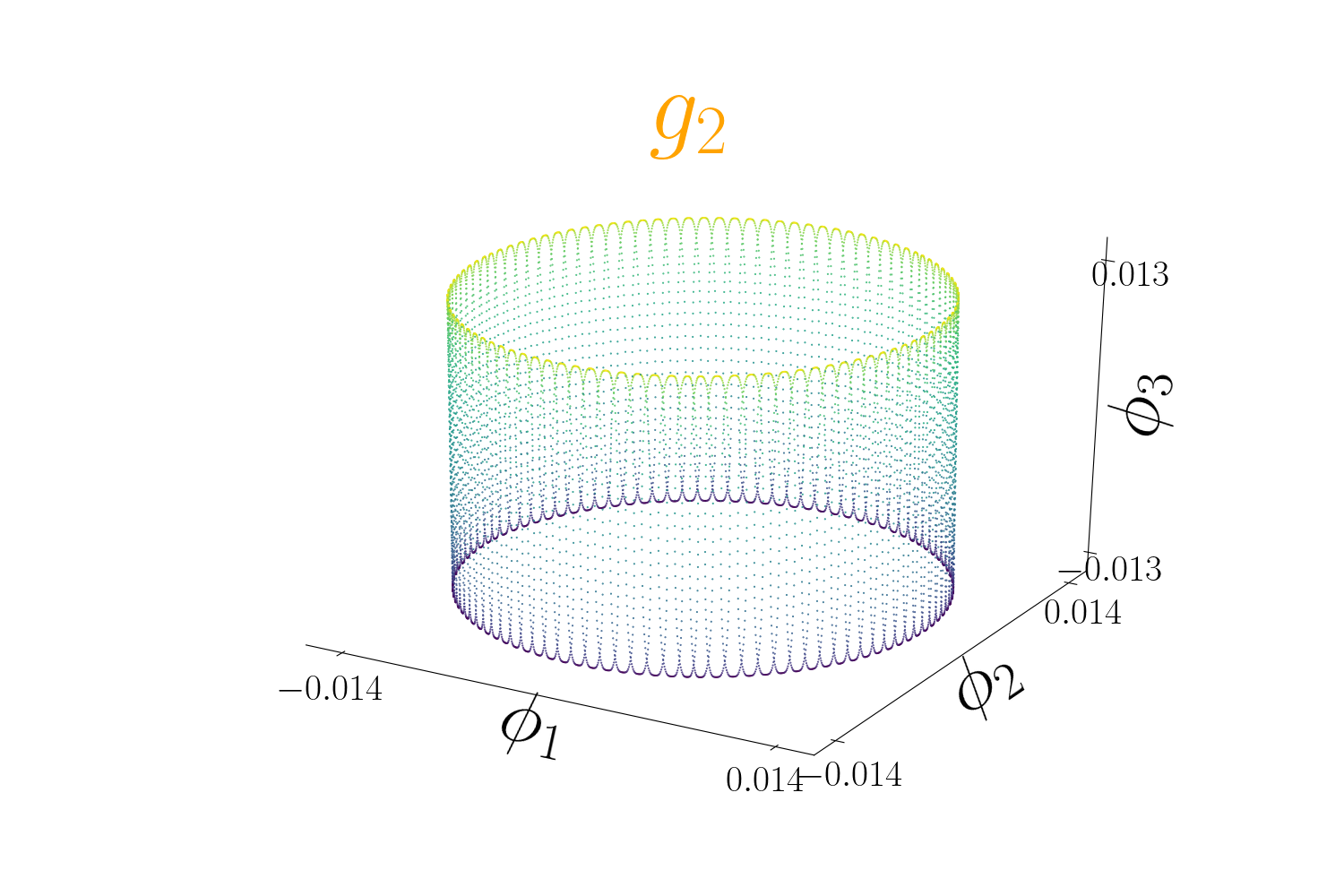}\label{fig:re-nonoise-embedding-g2}}
    \caption{PCA features and Diffusion Map embedding features of rigid ethanol data without noise.}
\end{figure}

Figure \ref{fig:re-nonoise-cosine}-\ref{fig:re-nonoise-watch} show the result of experiments on regid ethanol without any noise. We can tell from the result that over all 25 replicates, \tsalg~ successfully recover the true support, one function from each colinear group.

\begin{figure}
    \centering
    \subfloat[]{\includegraphics[width = 0.3\textwidth]{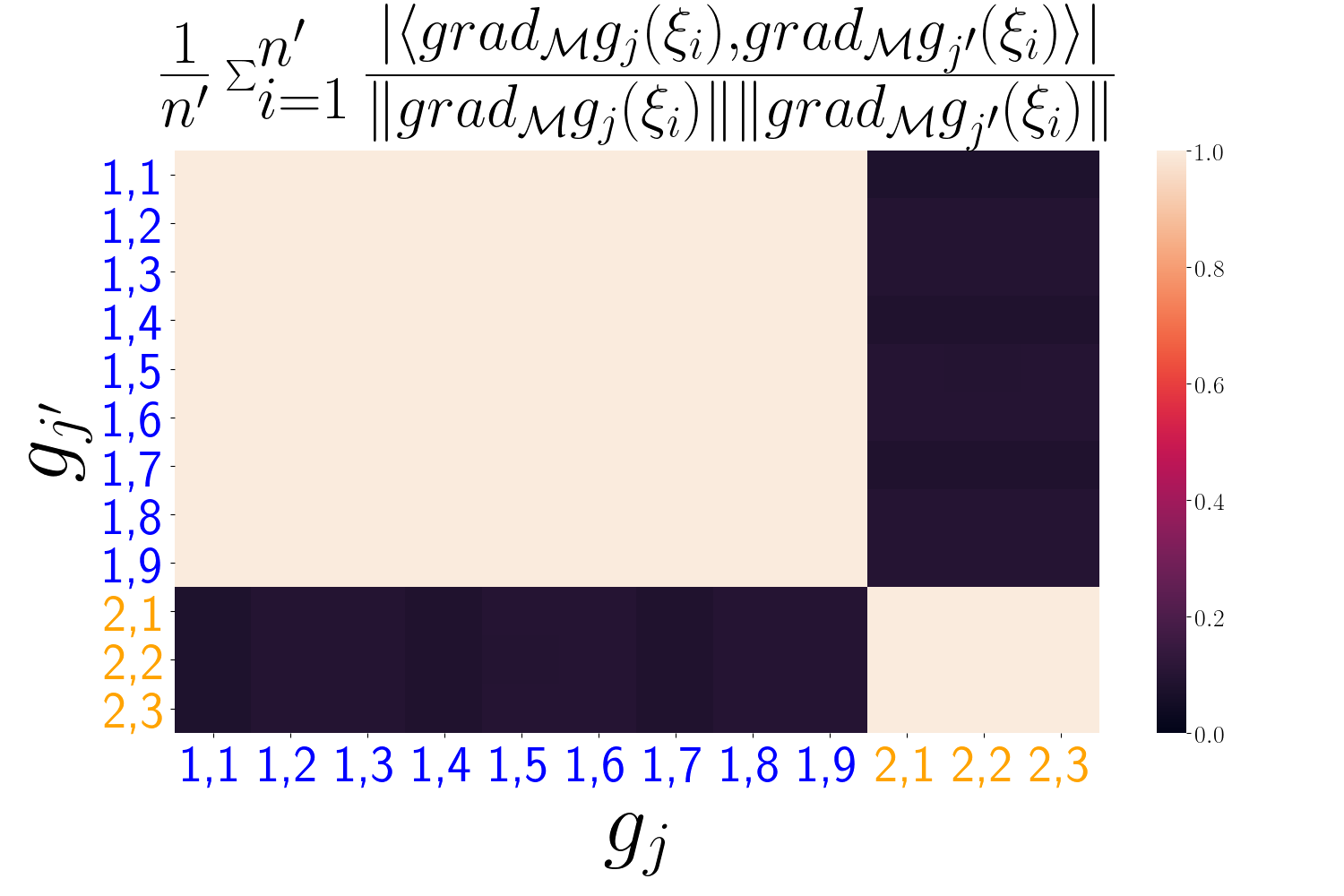}\label{fig:re-nonoise-cosine}}
    \subfloat[]{\includegraphics[width = 0.3\textwidth]{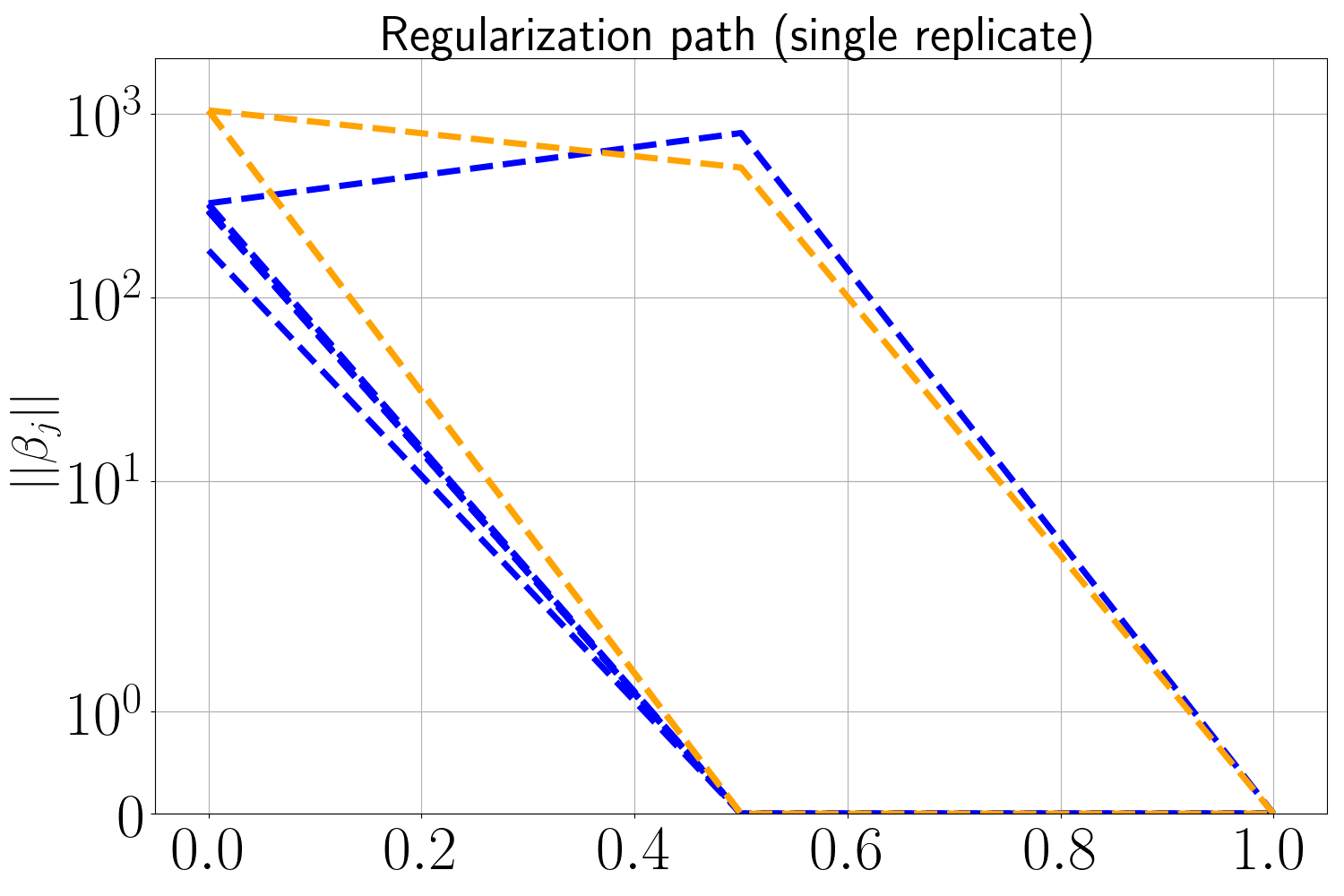}\label{fig:re-nonoise-regularization}}
    \subfloat[]{\includegraphics[width = 0.3\textwidth]{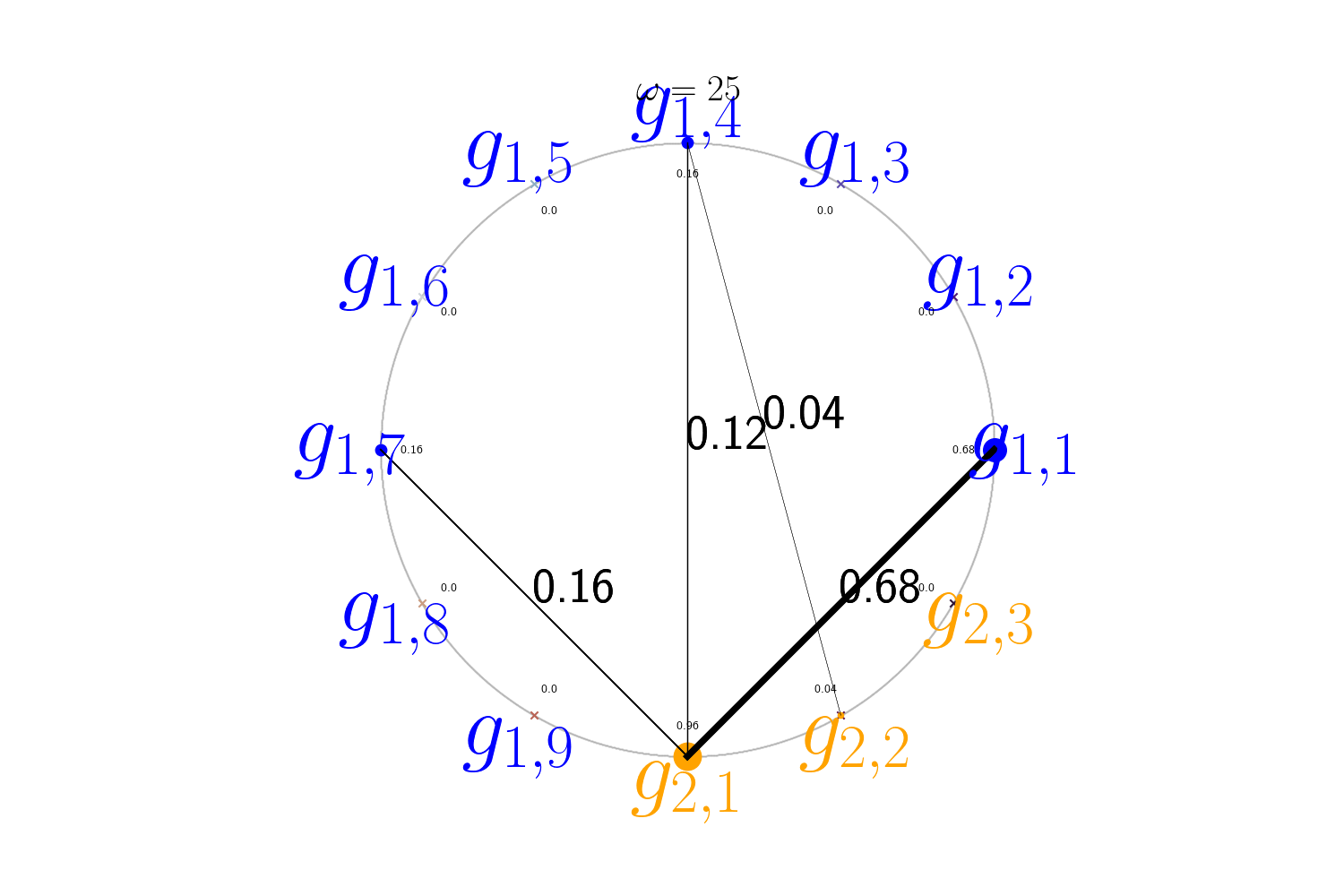}\label{fig:re-nonoise-watch}}
    \caption{Results of Rigid Ethanol Experiment with no noise. \textbf{Left:} cosine plots of dictionary functions, showing the existence of two groups of highly colinear functions. \textbf{Middle:} regularization path in one replicate. \textbf{Right:} The frequency of each pair of function selected in all 25 replicates. }
\end{figure}

With the increase in noise, we display the watch plot in figure \ref{fig:re-p00001-watch}-\ref{fig:re-p01-watch}. With the increase in the noise, it is possible that \tsalg~ do not recover the correct support. For example when noise level is $\sigma = 0.1$, in all replicates, \tsalg~ selects two functions in the same group. Interestingly, when we look at the embedding given by Diffusion Maps at this noise level, we observe that the torus topology is broken, as shown in figure \ref{fig:rigidethanol-highnoise-diffusion1} and \ref{fig:rigidethanol-highnoise-diffusion2}.

\begin{figure}
    \centering
    \subfloat[$\sigma=0.0001$]{\includegraphics[width = 0.23\textwidth]{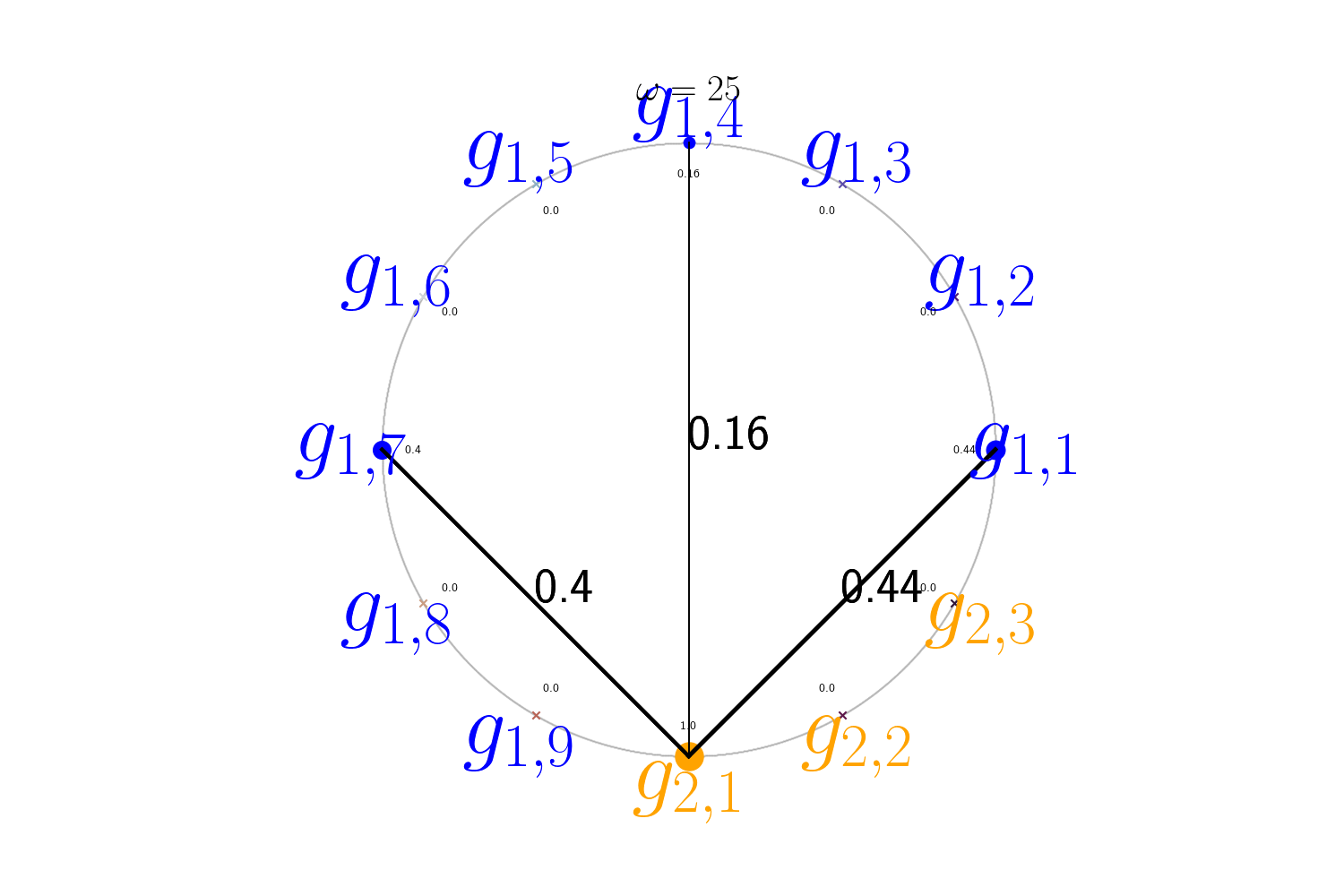}\label{fig:re-p00001-watch}}
    \subfloat[$\sigma=0.001$]{\includegraphics[width = 0.23\textwidth]{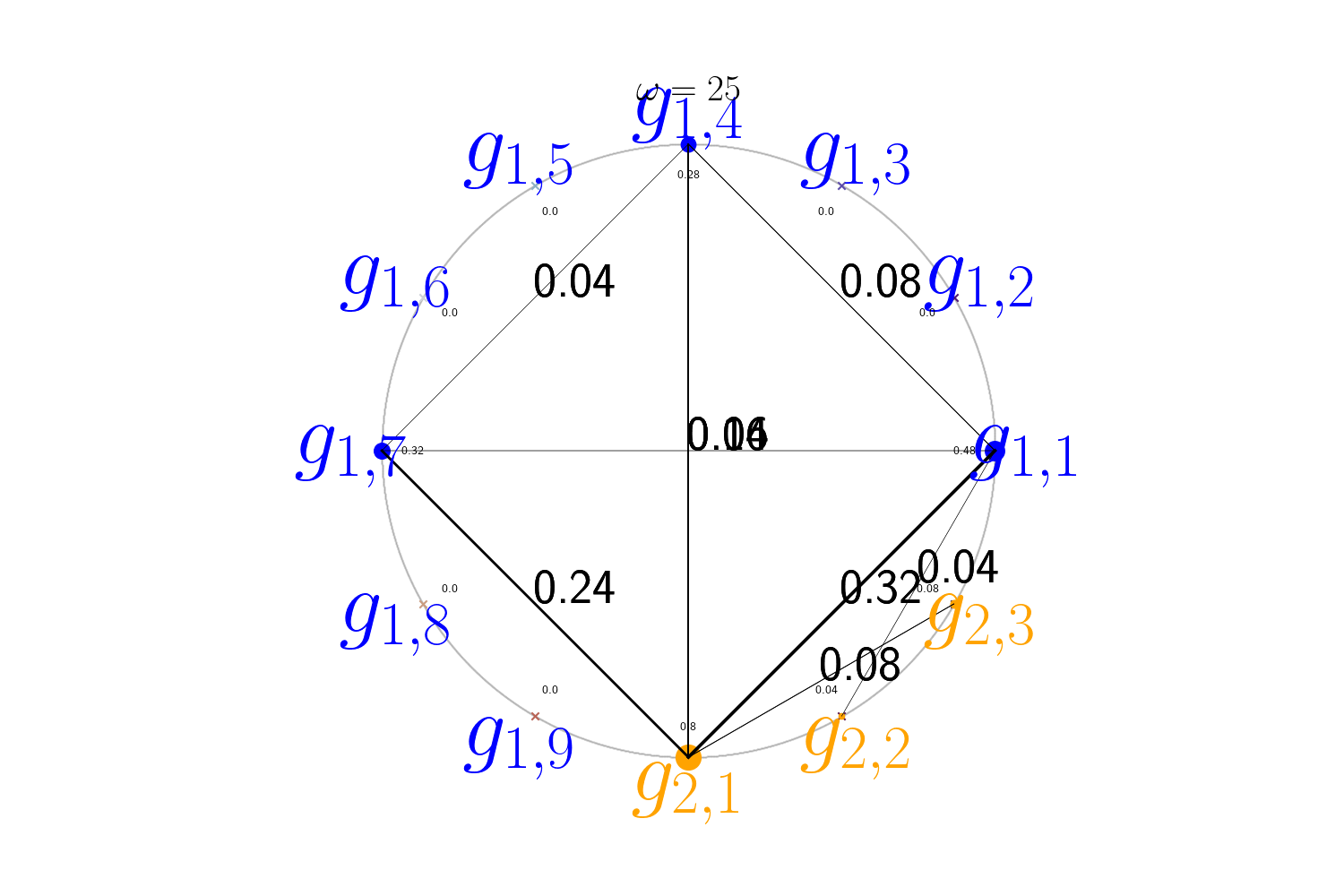}\label{fig:re-p0001-watch}}
    \subfloat[$\sigma=0.01$]{\includegraphics[width = 0.23\textwidth]{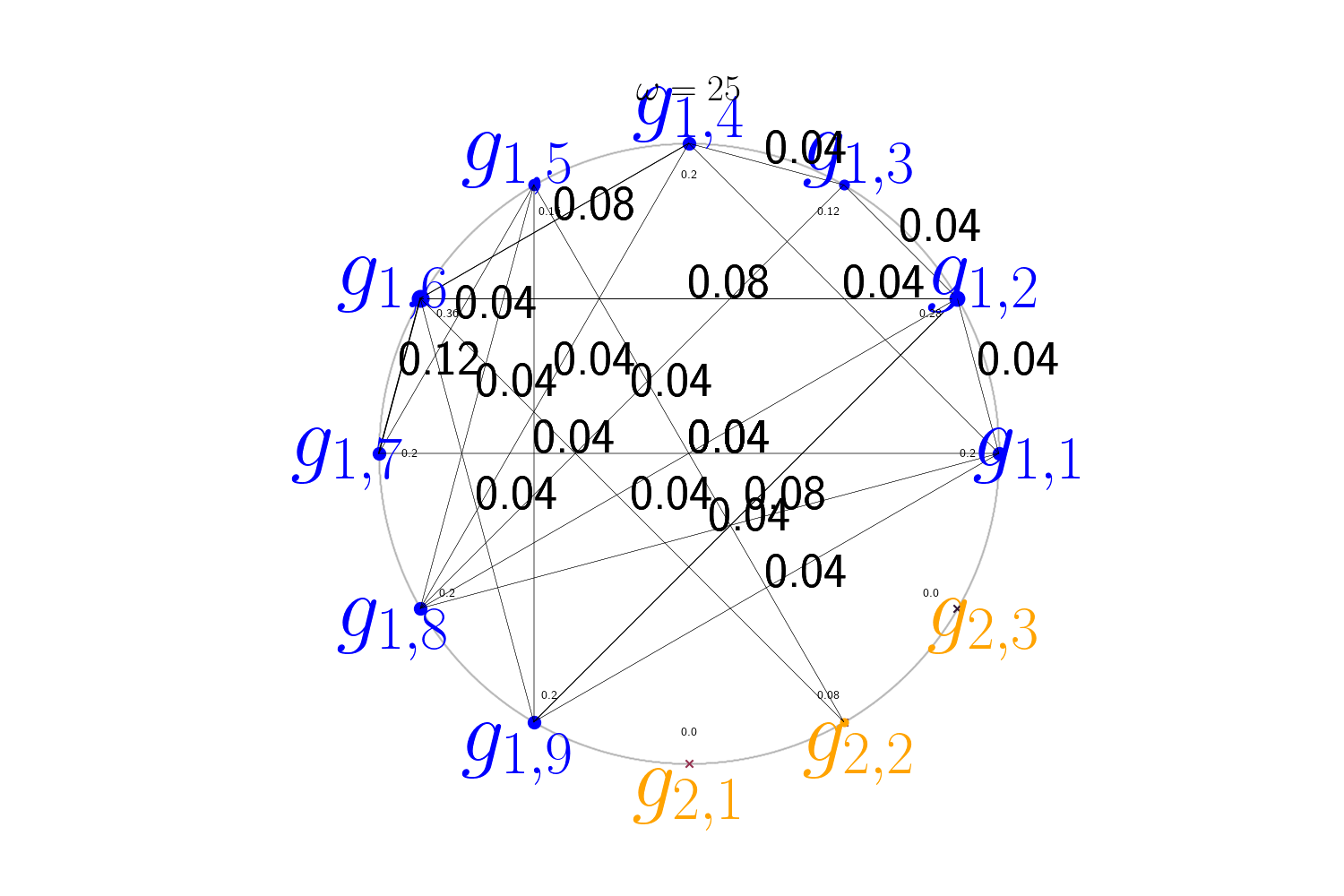}\label{fig:re-p001-watch}}
    \subfloat[$\sigma=0.1$]{\includegraphics[width = 0.23\textwidth]{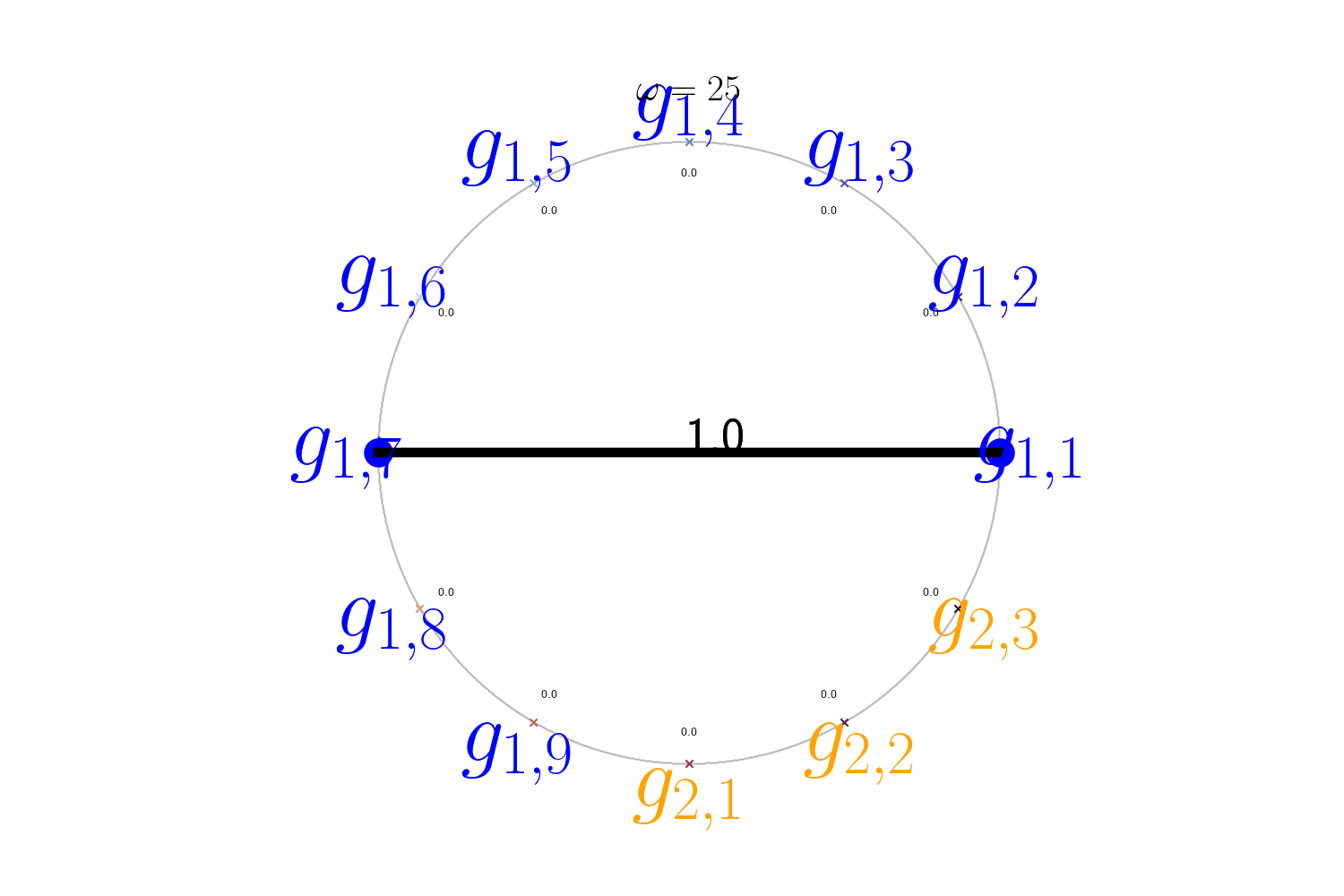}\label{fig:re-p01-watch}}
    \caption{Watch plot of support recovery frequencies under different noise levels. }
\end{figure}

\begin{figure}
    \centering
    \centering
    \subfloat[]{\includegraphics[width = 0.45\textwidth]{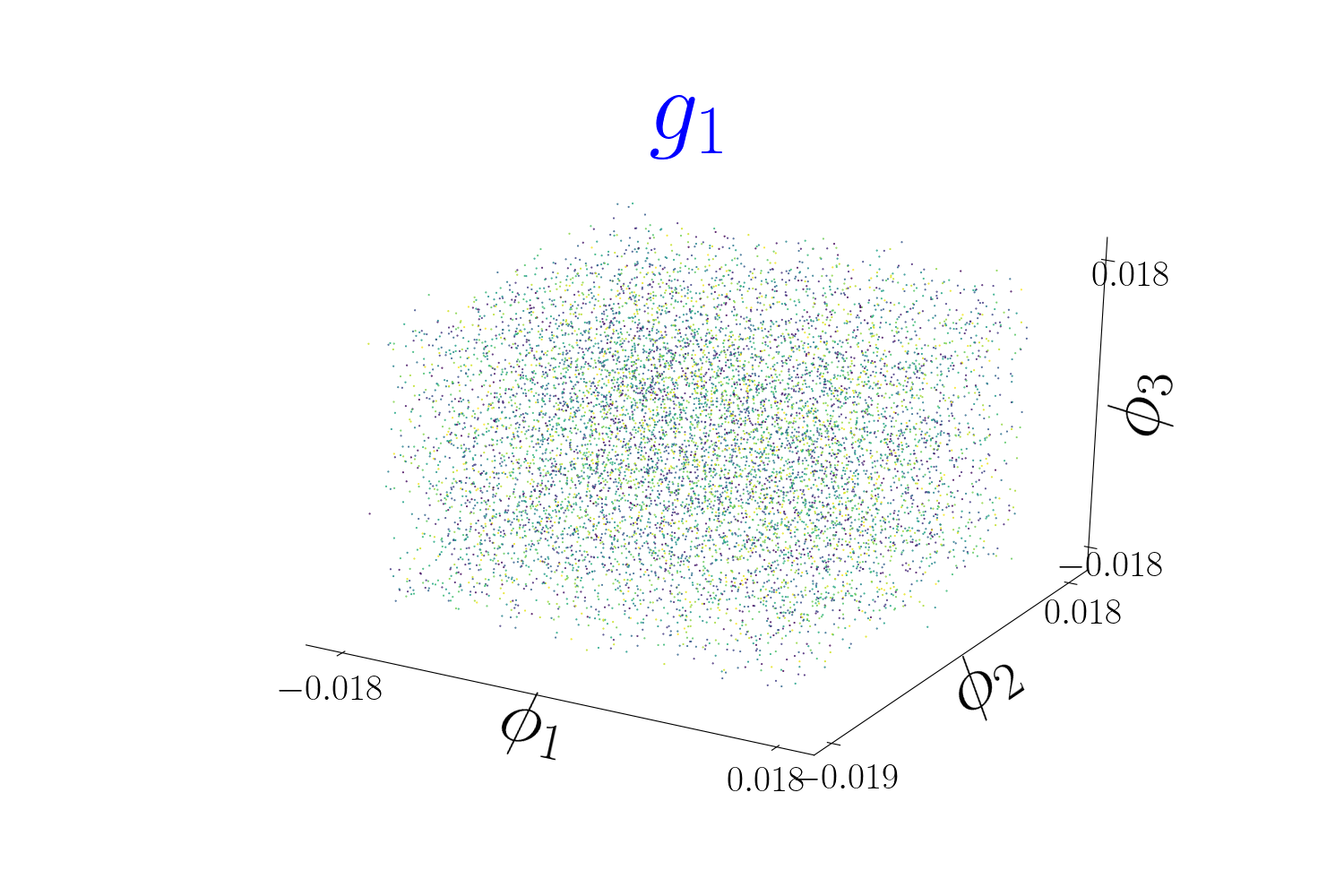}\label{fig:rigidethanol-highnoise-diffusion1}}
    \subfloat[]{\includegraphics[width = 0.45\textwidth]{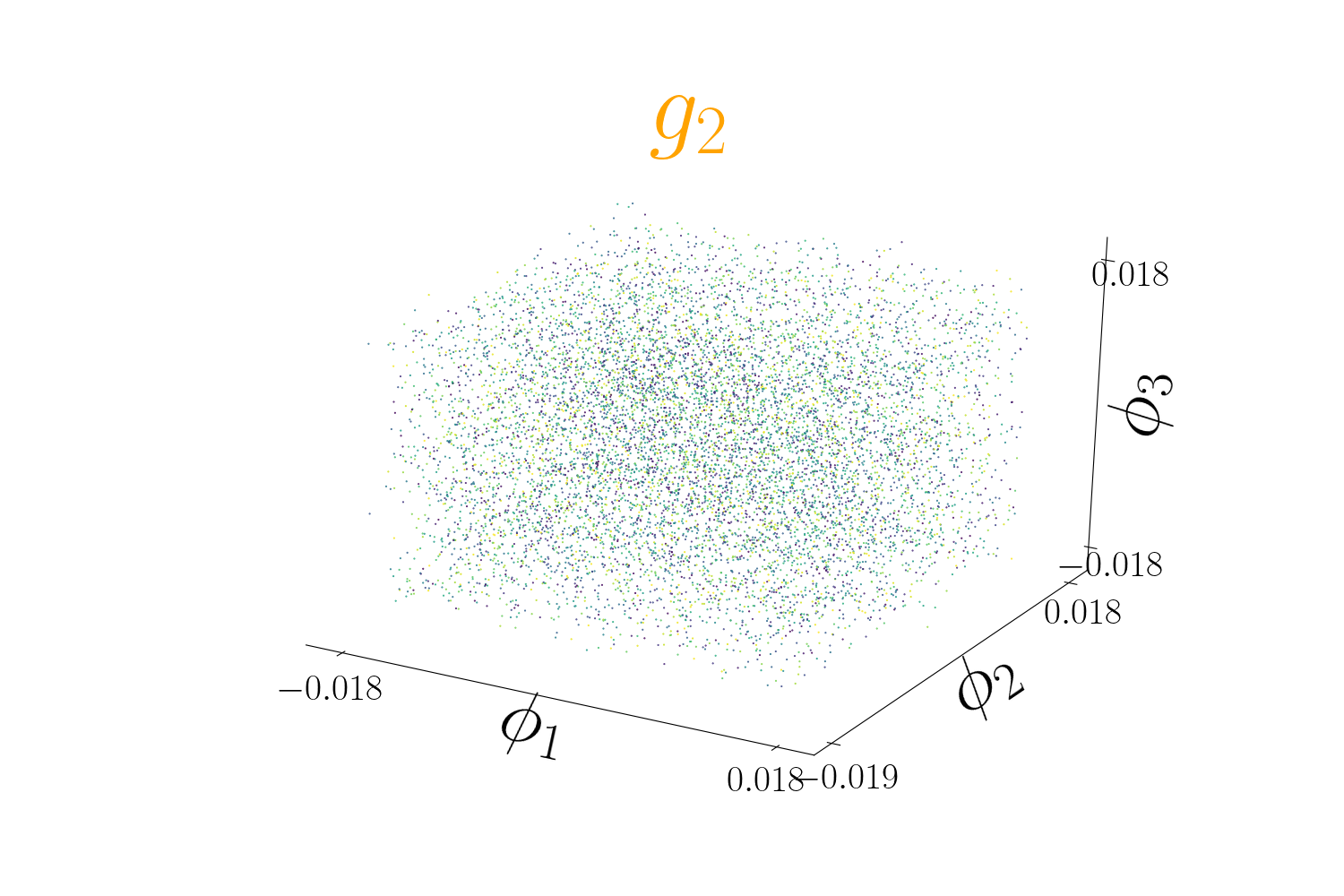}\label{fig:rigidethanol-highnoise-diffusion2}}
    \caption{Diffusion map embedding for synthetic rigid ethanol data. Data points are colored by the true torsion $g_1$ and $g_2$ respectively.}
\end{figure}

\subsection{Comparison with Embeddings}
The comparisons with Diffusion maps of Toluene are shown in the introduction in sectoin \ref{sec:intro}. Here we display some comparison of \tsalg~ with Diffusion maps on real MDS data, which are widely used for dimension reduction.
 Figure \ref{fig:ethanol-diffusion1} and \ref{fig:ethanol-diffusion2} shows that the two functions selected from the \tsalg~ indeed parametrize the structure of the data. As the values are roughly varying along with two circles of the torus.
 Figure \ref{fig:mal-diffusion1} and \ref{fig:mal-diffusion2} shows a pair of functions selected by \tsalg~. 

\begin{figure}
    \centering
    \subfloat[]{\includegraphics[width = 0.4\textwidth]{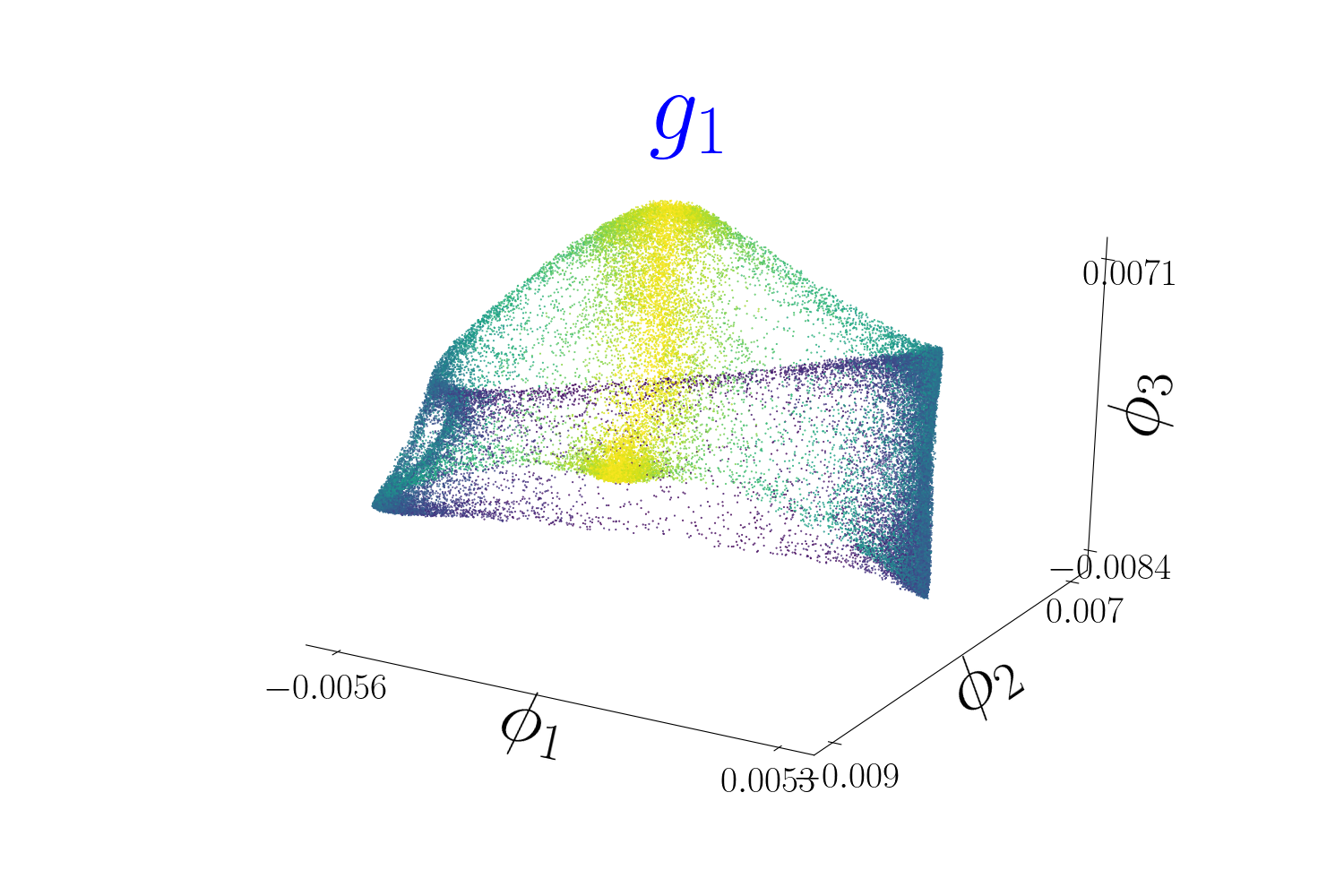}\label{fig:ethanol-diffusion1}}
    \subfloat[]{\includegraphics[width = 0.4\textwidth]{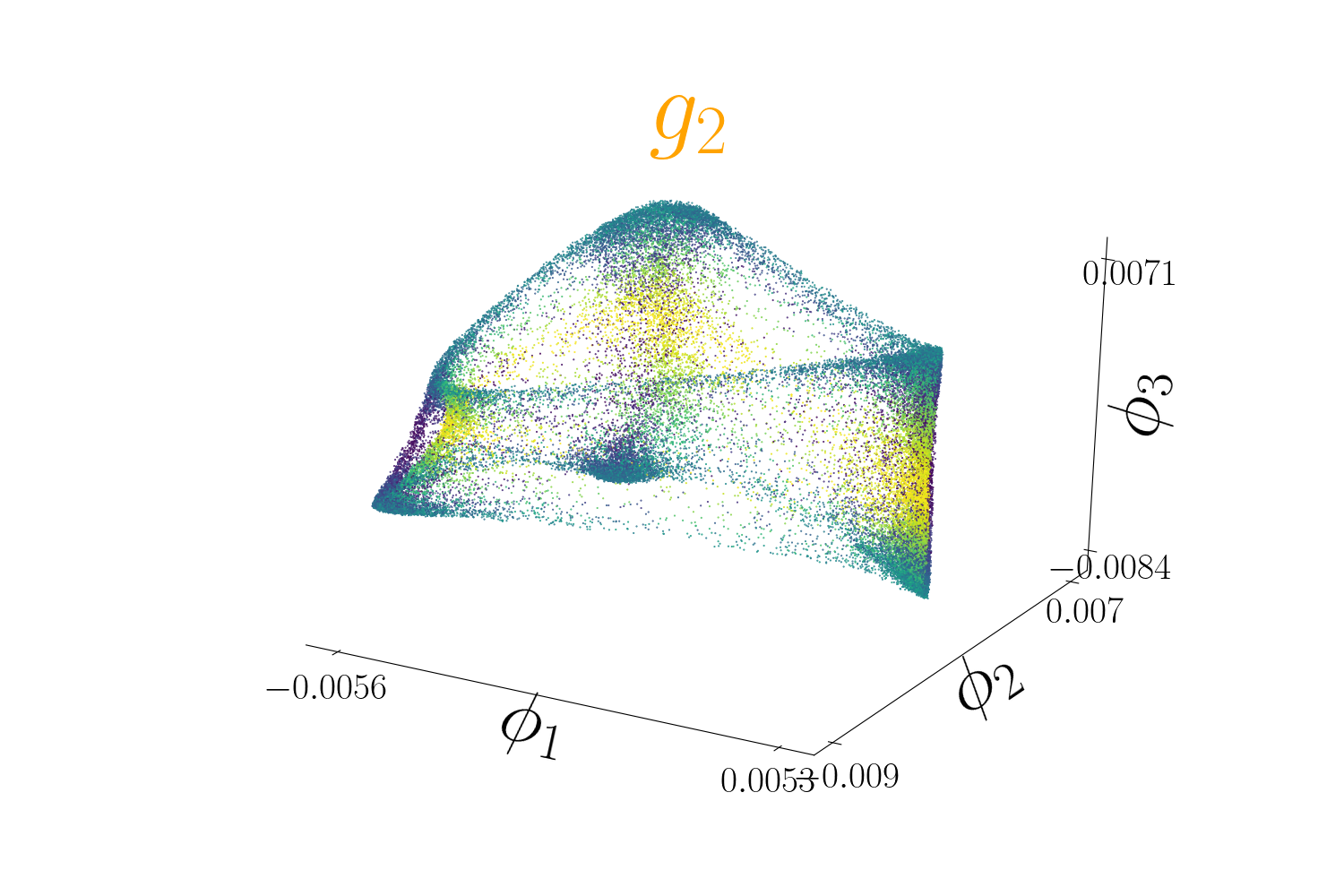}\label{fig:ethanol-diffusion2}}
    \caption{Diffusion map embedding for real ethanol data. Data points are colored by the two torsion functions $g_0,g_9$ found by \tsalg~ resepctively. }
\end{figure}

\begin{figure}
    \centering
    \subfloat[]{\includegraphics[width = 0.4\textwidth]{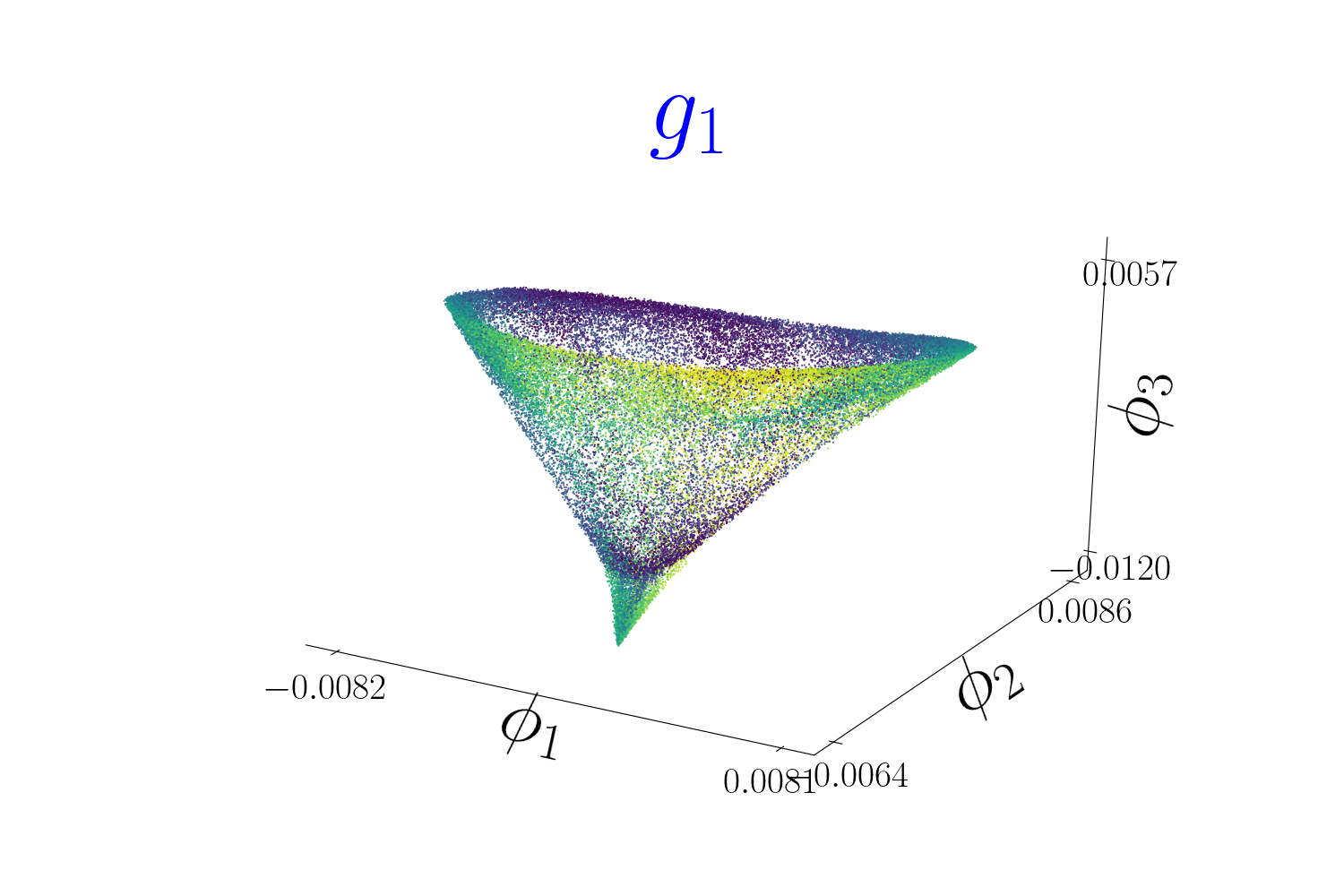}\label{fig:mal-diffusion1}}
    \subfloat[]{\includegraphics[width = 0.4\textwidth]{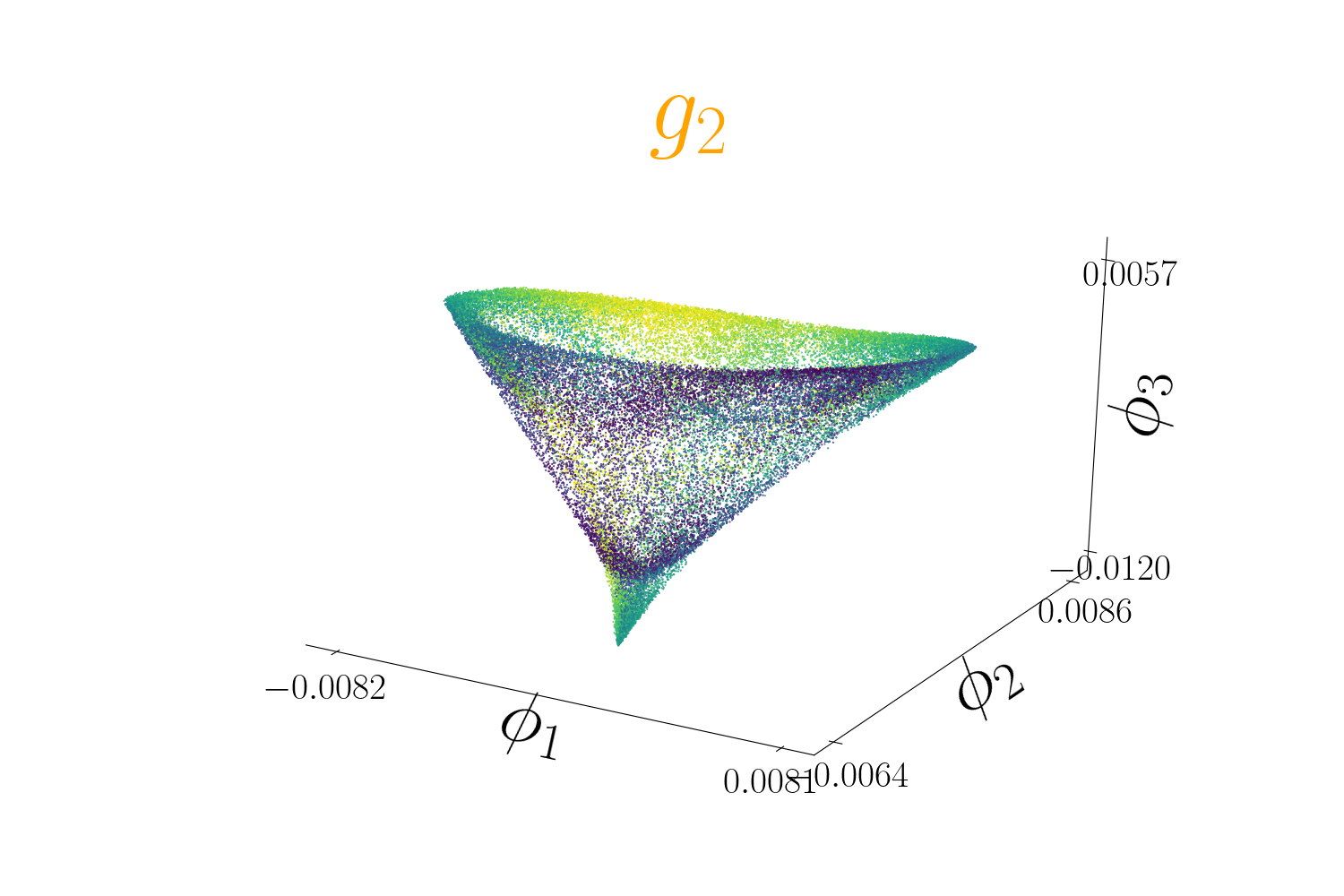}\label{fig:mal-diffusion2}}
    \caption{Diffusion map embedding for Malonaldehyde data. Data points are colored by the two torsion functions found by \tsalg~ resepctively. }
\end{figure}

\end{document}